\documentclass[letterpaper]{article} 
\usepackage[draft]{aaai2026}  
\usepackage{times}  
\usepackage{helvet}  
\usepackage{courier}  
\usepackage[hyphens]{url}  
\usepackage{graphicx} 
\usepackage{natbib}  
\usepackage{caption} 
\usepackage{algorithm}
\usepackage{algorithmic}
\usepackage{amsmath}
\usepackage{amssymb}
\usepackage{mathtools}
\usepackage{amsthm}
\usepackage{booktabs}   
\theoremstyle{plain}
\newtheorem{theorem}{Theorem}[section]

\theoremstyle{definition}

\theoremstyle{remark}

\newcommand{\bi}{\begin{itemize}}
\newcommand{\ei}{\end{itemize}}

\newcommand{\FIG}{Fig.~}

\newcommand{\EQN}{Eq.~}
\newcommand{\THM}{Thm.~}

\newcommand{\Dcal}{\mathcal{D}}

\usepackage{subcaption}

\usepackage{newfloat}
\usepackage{listings}
\DeclareCaptionStyle{ruled}{labelfont=normalfont,labelsep=colon,strut=off} 
\floatstyle{ruled}
\newfloat{listing}{tb}{lst}{}
\floatname{listing}{Listing}
\title{Risk-Sensitive Exponential Actor Critic}
\author{
    Alonso Granados,
    Jason Pacheco}
\affiliations{
    Department of Computer Science, University of Arizona\\
    \{alonsog, pachecoj\}@arizona.edu

}

\usepackage{bibentry}

\begin{document}

\maketitle

\begin{abstract}
Model-free deep reinforcement learning (RL) algorithms have achieved tremendous success on a range of challenging tasks. However, safety concerns remain when these methods are deployed on real-world applications, necessitating risk-aware agents.  A common utility for learning such risk-aware agents is the entropic risk measure, but current policy gradient methods optimizing this measure must perform high-variance and numerically unstable updates.  As a result, existing risk-sensitive model-free approaches are limited to simple tasks and tabular settings. In this paper, we provide a comprehensive theoretical justification for policy gradient methods on the entropic risk measure, including on- and off-policy gradient theorems for the stochastic and deterministic policy settings.  Motivated by theory, we propose risk-sensitive exponential actor-critic (rsEAC), an off-policy model-free approach that incorporates novel procedures to avoid the explicit representation of exponential value functions and their gradients, and optimizes its policy w.r.t.~the entropic risk measure. We show that rsEAC produces more numerically stable updates compared to existing approaches and reliably learns risk-sensitive policies in challenging risky variants of continuous tasks in MuJoCo.
\end{abstract}


\section{Introduction}
Model-free deep reinforcement learning (RL) algorithms have been successful at learning complex behaviors using only interactions with the environment \citep{mnih2015human, haarnoja2018soft, gu2017deep}. However, interacting with real-world
applications, such as autonomous driving \citep{mavrin2019distributional}, robotics \citep{nass2019entropic} and finance \citep{artzner1999coherent, lai2011mean}, can lead to catastrophic events, motivating a need for risk-aware agents in decision-making problems. Risk-sensitive RL aims to learn policies that maximize performance metrics that incorporate a penalty for the variability of the return --- either due to intrinsic uncertainty in the transition dynamics or randomness in the reward signal. 

Risk-sensitive RL has been studied through various risk criteria: variance-related risk measures \citep{tamar2012policy, la2013actor}, reward-volatility risk \citep{ijcai2020p632, zhang2021mean}, Gini-deviation \citep{luo2023alternative},  distortion risk \citep{dabney2018implicit}, and conditional value at risk (CVaR) \citep{morimura2010nonparametric, chow2014algorithms}. In this paper, we consider risk-sensitive RL with the entropic risk measure. This metric incorporates risk into its objective via the exponential utility function \citep{howard1972risk}, resulting in a nonlinear Bellman equation. In contrast to the standard RL problem, sample-based estimates of the expectation operation are biased. Therefore, most of the work under this framework has been in the context of MDP control, where the transition dynamics and reward model are known \citep{garcia2015comprehensive}. The exponential Bellman equations provide a model-free approach that leads to improved regret bounds for risk-sensitive RL \cite{fei2021exponential}. In the approximate setting, we can minimize the exponentiated TD error, which accelerates the learning process and produces robust policies to model perturbations \cite{noorani2023exponential}. However, the estimation of these functions is numerically unstable.

Policy gradient algorithms improve their policy by updating it in the direction of the performance gradient. The fundamental result for these methods is the policy gradient theorem \citep{sutton1999policy}, which allows for the estimation of the gradient using only sampled trajectories. In risk-sensitive RL, an estimate of the gradient can be derived using the likelihood ratio trick \citep{nass2019entropic}. However, the estimator may suffer from high variance as its computation is w.r.t. the full trajectory, and it is scaled by the exponentiated return, which can be numerically unstable. These numerical instabilities are further exacerbated when introducing function approximators \citep{maei2009convergent}. 


We now present our contributions to address these shortcomings. First, we derive risk-sensitive policy gradient theorems for the on-policy setting and approximations for the off-policy setting. These results circumvent the need to estimate a gradient for the full trajectory. Second, we study the instabilities that emerge from learning a critic network that approximates the exponentiated return, and propose a representation that avoids its explicit computation by allowing for computation in log-domain. Additionally, we consider a stabilizing mechanism for the critic gradient, based on mini-batch normalization and clipping. Finally, we present risk-sensitive exponential actor-critic (rsEAC), a practical off-policy algorithm that optimizes the entropic risk measure where the actor is optimized using our risk-sensitive off-policy gradient, and the critic incorporates our stabilizing mechanisms to estimate the entropic risk measure for the current policy. We evaluate our algorithm on complex continuous tasks where risk-aversion can be verified, and demonstrate that our method can learn risk-sensitive policies with high-return. All code is included in the supplemental materials.

\section{Background}
Consider a finite horizon Markov decision process (MDP) which is composed of: a state space $\mathcal{S}$, an action space $\mathcal{A}$, an initial state distribution $p_1(s_1)$, a transition dynamics distribution $p(s_{t+1}|s_t, a_t)$ that satisfies the Markov property $p(s_{t+1}|s_t, a_t, 
\dots, s_1,a_1) = p(s_{t+1}|s_t, a_t)$, and a reward function $r :\mathcal{S} \times \mathcal{A} \to \mathbb{R}$, which we abbreviate as $r_t := r(s_t, a_t)$. We assume that the agent does not have access to the transition dynamics and the reward function. The agent selects actions based on a policy denoted by $\pi_\theta(a_t|s_t)$, a conditional distribution over $\mathcal{A}$ parameterized by $\theta \in \mathbb{R}^n$. The distribution over a trajectory $\tau = (s_1,a_1,\dots,s_T,a_t,s_{T+1})$ for a policy $\pi_\theta$ is given by \mbox{$p_\pi(\tau) = p(s_1) \prod_{t=1}^T p(s_{t+1} \mid s_t,a_t) \pi_\theta(a_t \mid s_t)$}. Value functions are defined to be the expected cumulative rewards, $V_{\pi_\theta}(s) = \mathbb{E}_{p_{\pi}(\tau)}[\sum_{t=1}^T r_t|s_1=s]$ and $Q_{\pi_\theta}(s, a) = \mathbb{E}_{p_{\pi}(\tau)}[\sum_{t=1}^T r_t|s_1=s, a_1= a]$. We denote the distribution of state $s'$ after transitioning for $t$ time steps from state s by $p(s \to s', t, \pi)$. We also denote the state distribution by $\rho_\pi(s') :=\int_S \sum_{t=1}^T p_1(s)p(s \to s', t, \pi)\,\mathrm{d}s$. The standard objective in RL is to find a policy that maximizes expected return $J(\pi_\theta) =  \mathbb{E}_{p_{\pi}(\tau)}[\sum_{t=1}^T r_t]$. 
\subsection{Policy Gradient Algorithms}

Policy gradient algorithms update the parameters of their policy in the direction of the performance gradient $\nabla_\theta J(\pi_\theta)$. For stochastic policies $\pi_\theta$, the gradient can be derived from the policy gradient theorem \citep{sutton1999policy},
\begin{equation}\label{eq:stochastic_PG}
\nabla_\theta J(\pi_\theta) = \int_S \rho_{\pi}(s) \int_A \nabla_\theta \pi_\theta (a|s)Q_{\pi_\theta}(s,a) \,\mathrm{d}a \,\mathrm{d}s.
\end{equation}
Although the state distribution $\rho_\pi(s)$ depends on the policy parameters, the gradient surprisingly does not involve the derivative of the state distribution. The work by \citet{silver2014deterministic} extended this result to deterministic policies $\mu_\theta$, yielding an analogous gradient that can be estimated more efficiently than the stochastic policy gradient as it lacks the integral over actions:
\begin{equation}
\nabla_\theta J(\mu_\theta) = \!\! \int_S \!\!\rho_{\mu}(s) \nabla_\theta \mu_\theta (s) \nabla_a Q_{\mu_\theta}(s,a)|_{a= \mu_\theta(s)}  \,\mathrm{d}s.
\end{equation}
These two results can be extended to off-policy trajectories sampled by a different behavior policy $b(a\mid s) \neq \pi_\theta (a\mid s)$. The performance objective is modified to be the value function of the target policy $V_{\pi_\theta}$, weighted by the state distribution of the behavior policy, $J_b(\pi_\theta) = \int_S \rho_{b}(s) V_{\pi_\theta}(s) \,\mathrm{d}s$.
The gradient of this objective can be approximated by an off-policy variant of the policy gradient theorem \citep{degris2012off}:
\begin{equation}
\nabla_\theta J_b(\pi_\theta) \approx \int_{\mathcal{S}} \int_{\mathcal{A}} \rho_{b}(s) \nabla_\theta \pi_\theta (a|s)Q_{\pi}(s,a) \,\mathrm{d}a \,\mathrm{d}s.
\end{equation}
This approximation drops a term that depends on $\nabla_\theta Q_{\pi_\theta}(s,a)$, because it can be difficult to estimate in an off-policy fashion. However, this approximation still produces policy improvement when using a tabular representation for the policy \citep{degris2012off}.

\subsection{Risk-Sensitive Objective}\label{sec:entropic_risk}
In risk-sensitive RL with the entropic risk measure, we aim to find a policy that maximizes:
\begin{equation}\label{eq:log_entropic}
J^\beta(\pi_\theta) =  \frac{1}{\beta} \log \mathbb{E}_{p_\pi(\tau)} \left[e^{\beta\sum_{t=1}^T r_t} \right],
\end{equation}
where $\beta \in \mathbb{R}$ controls risk preference. This objective is closely related to mean-variance RL~\citep{mannor2011mean} --- it admits the Taylor expansion \mbox{$\mathbb{E}_{p_\pi(\tau)}[\sum_t r_t] +\frac{\beta}{2}\text{Var}_\pi(\sum_t r_t) +O(\beta^2)$} \citep{mihatsch2002risk}.
Thus, \EQN\eqref{eq:log_entropic} reduces to the standard (\emph{risk-neutral}) RL objective for $\beta \to 0$. In addition, $\beta > 0$ induces \emph{risk-seeking} policies and $\beta <0$ induces \emph{risk-averse} policies. We now define soft-value functions as the entropic risk measure of cumulative rewards, 
\begin{gather}
\!\!V_{\pi_\theta}^\beta(s) =\! \frac{1}{\beta}\! \log \mathop{\mathbb{E}}_{p_\pi(\tau)} \left[e^{\beta\sum_{t=1}^T r_t} | s_1\! = s \right]\! ,\! \quad \!\!\!\! \\ Q_{\pi_\theta}^\beta(s,a)\! = \!\frac{1}{\beta} \log \mathop{\mathbb{E}}_{p_\pi(\tau)} \left[ e^{\beta \sum_{t=1}^T r_t}\mid \! s_1 \! = \! s, \! a_1 \! = \! a\right].\label{eq:bellman_V}
\end{gather}
These value functions are recursively associated via Bellman-style backup equations:
\begin{gather}
V_{\pi_\theta}^\beta(s_t) \! =\!  \frac{1}{\beta} \log \mathop{\mathbb{E}}_{\pi(a_t|s_t)} \left[e^{\beta Q_{\pi_\theta}^\beta(s_t,a_t)}\right]\!,\!\!\! \quad \\
 Q_{\pi_\theta}^\beta(s_t,a_t)\! = \! r_t\! + \!\frac{1}{\beta} \log \mathop{\mathbb{E}}_{p(s_{t+1} \mid s_t, a_t)} \left[e^{\beta V_{\pi_\theta}^\beta(s_{t+1})}\right].\label{eq:bellman_Q}
\end{gather}
The main challenge of risk-sensitive RL lies in its nonlinear Bellman equations. In contrast to the standard RL problem, sample-based estimates of the value functions are biased due to the $\log(\cdot)$ operator. Hence, this line of work is dominated by the MDP control framework, where the transition dynamics and the reward model are available to the agent \citep{garcia2015comprehensive}.

\section{Gradients of the Entropic Risk Measure}\label{sec:theoretical}
The entropic risk measure commonly invoke the likelihood ratio trick to obtain an estimate of the gradient of the performance objective \citep{nass2019entropic}. However, these estimates suffer from high variance, as they are computed w.r.t. full trajectories, and are proportional to the exponentiated return, which can lead to numerical instabilities during learning. In this section, we address the first issue by extending the policy gradient framework to risk-sensitive policies under the entropic risk measure. First, we introduce policy gradient theorems under the on-policy setting for both stochastic (\THM\ref{thm:stoch_pg}) and deterministic (\THM\ref{thm:det_pg}) policies. Then, we extend these two theorems to the off-policy setting by approximating these gradients, and demonstrate policy improvements for these approximations when using a tabular representation of the policy (\THM\ref{thm:policy_improvement}). All proofs can be found in Appendix A.





\subsection{On-policy Risk-Sensitive Policy Gradients}
We begin by defining the exponential twisting of dynamics and policy \citep{ijcai2021p316} as: $p^*(s'|s,a) \!\propto p(s'|s,a) e^{\beta V_{\pi_\theta}^\beta(s')}$, and $\pi_\theta^*(a|s) \propto \pi_\theta(a|s) e^{\beta Q_{\pi_\theta}^\beta(s,a)},$ respectively, and the state distribution for dynamics $p^*(s'|s,a)$ and policy $\pi^*_\theta(a|s)$ as $\rho_\pi^*(s)$. We now provide our risk-sensitive policy gradient theorem for stochastic policies.
\renewcommand\thetheorem{1}
\begin{theorem}\label{thm:stoch_pg}
(Risk-Sensitive Stochastic Policy Gradient Theorem). The gradient of $J^\beta(\pi_\theta)$ w.r.t. $\theta$ is given by:
\begin{gather}
\frac{1}{\beta} \int_S \rho_{\pi}^*(s) \int_A \nabla_\theta \pi_\theta (a|s)e^{\beta (Q_{\pi_\theta}^\beta(s,a)- V_{\pi_\theta}^\beta(s))} \,\mathrm{d}a \,\mathrm{d}s.
\end{gather}
\end{theorem}
This gradient has two key differences from its risk-neutral version in \EQN\eqref{eq:stochastic_PG}.  First, it is weighted by the state distribution of exponential twisted dynamics and policy, which tend to be optimistic or pessimistic, depending on $\beta$.  Second, the action-value function is substituted by $e^{\beta (Q_{\pi_\theta}^\beta(s,a) - V_{\pi_\theta}^\beta(s))}$, which can lead to numerical instabilities during learning due to the exponentiated value functions. For deterministic policies, we analogously define the state distribution $\rho_\mu^*(s)$. We now present our risk-sensitive policy gradient for deterministic policies $\mu_\theta$.
\renewcommand\thetheorem{2}
\begin{theorem}\label{thm:det_pg}
(Risk-Sensitive Deterministic Policy Gradient Theorem). The gradient of~$J^\beta(\mu_\theta)$ w.r.t. $\theta$ is given by:
\begin{gather}
\int_S \rho_{\mu}^*(s) \nabla_\theta \mu_\theta(s) \nabla_a Q_{\mu_\theta}^\beta(s,a)|_{a = \mu_\theta(s)} \,\mathrm{d}s.
\end{gather}
\end{theorem}
In particular, we notice that there is no need for an exponential twisted policy. The deterministic gradient is more convenient for two reasons: first, it avoids the integral over actions resulting in a more efficient approach, and second, it avoids the evaluation of an exponential term, which can be numerically unstable during learning.

\subsection{Off-policy Risk-Sensitive Policy Gradients}
We now consider methods that optimize a risk-sensitive policy from trajectories generated by an arbitrary behavior policy $b(a\mid s)$. To achieve this, we consider the soft-value functions of a target deterministic policy weighted by the state distribution of the behavior policy as our objective, $J_b^\beta(\mu_\theta) = \int_S \rho_{b}(s) V_{\mu_\theta}^\beta(s) \,\mathrm{d}s,$ where $\rho_b$ is the state distribution under the behavior policy.  Similarly to \citep{degris2012off}, we propose an approximation of the gradient $\nabla_\theta J_b^\beta(\mu_\theta)$:
\begin{equation}\label{eq:risk_off_gradient}
\int_{\mathcal{S}} \rho_{b}(s) \nabla_\theta \mu_\theta (s) \nabla_a Q_{\mu_\theta}^\beta(s,a)|_{a=\mu_\theta(s)} \,\mathrm{d}s \coloneqq g(\mu_\theta).
\end{equation}
This approximation drops a term that depends on $\nabla_\theta Q_{\mu_\theta}^\beta(s,a)$. We show that this approximation results in policy improvement when using a tabular representation for the policy.
\renewcommand\thetheorem{3}
\begin{theorem}\label{thm:policy_improvement}
(Deterministic Policy Improvement). Given a policy $\mu$ with parameters $\theta$ and a linear tabular representation, let $\theta' = \theta +\alpha g(\mu_\theta)$. Then there exist an $\epsilon$ such that for all $\alpha < \epsilon$:
\begin{gather}
V_{\mu_{\theta}}^\beta(s) \leq V_{\mu_{\theta'}}^\beta(s) \qquad \forall s \in S.
\end{gather}
\end{theorem}
For completeness, we consider an equivalent result for an off-policy performance objective that uses a stochastic policy $\pi_\theta$ as its target policy in Appendix A.

\section{Risk-Sensitive Exponential
Actor-Critic}\label{sec:rseac}
\begin{figure*}[t]
    \centering
     \begin{subfigure}[b]{0.26\textwidth}
         \includegraphics[width=\textwidth]{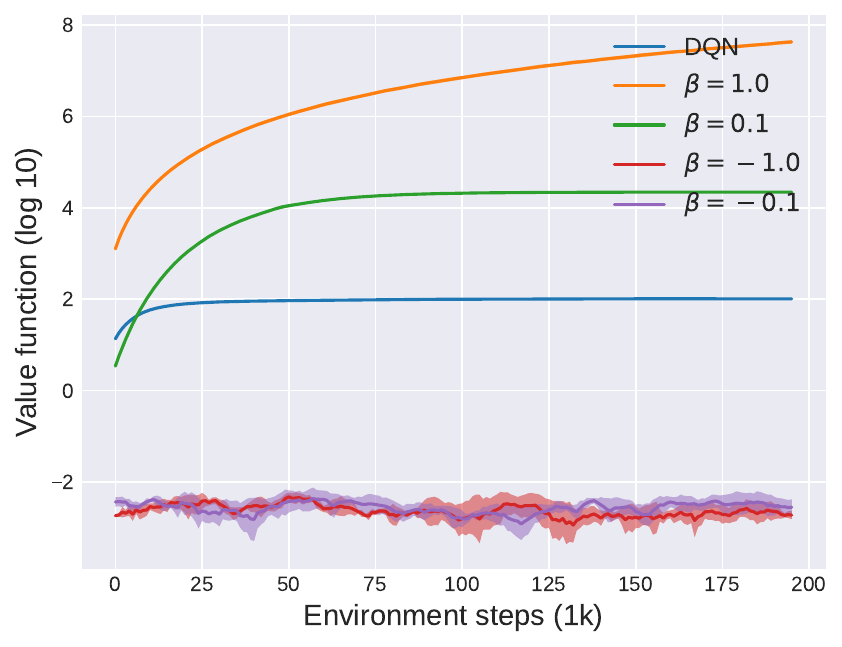}
         \caption{Average Initial Value}
         \label{fig:unstable_value}
     \end{subfigure}
     \begin{subfigure}[b]{0.26\textwidth}
         \includegraphics[width=\textwidth]{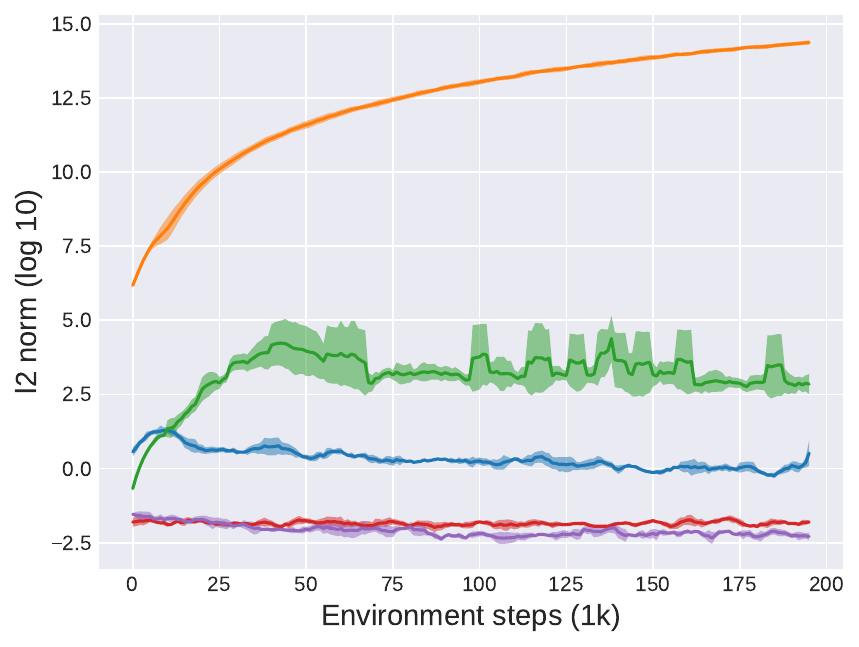}
         \caption{L2 Gradient Norm}
         \label{fig:unstable_gradient}
     \end{subfigure}
     \begin{subfigure}[b]{0.26\textwidth}
         \includegraphics[width=\textwidth]{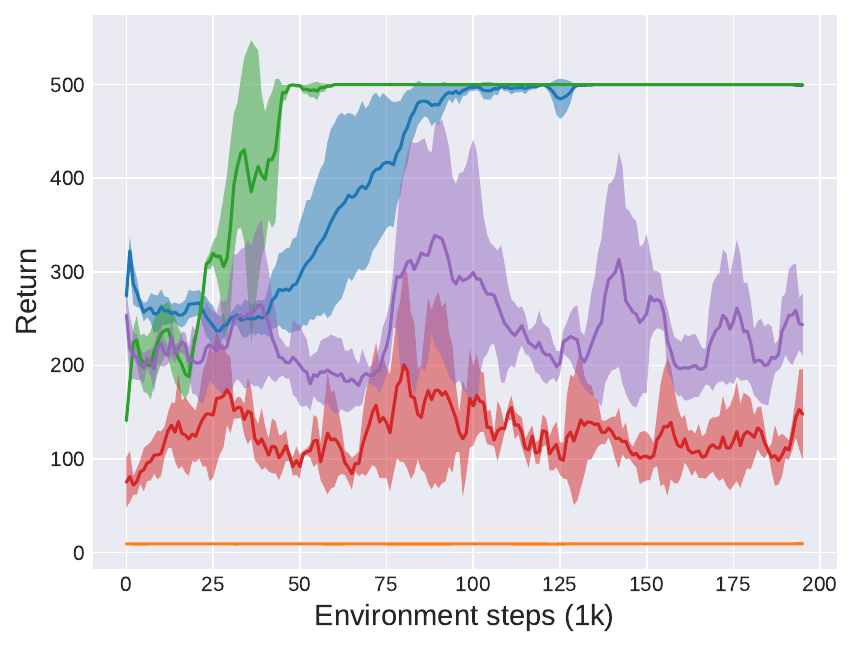}
         \caption{Average Performance}
         \label{fig:unstable_performance}
     \end{subfigure} 
    
     \caption{\textbf{Instabilities in value functions and gradients.} We train value functions for DQN and $Z_\psi(s_t,a_t)$ for a range of $\beta$ settings ($-1, -0.1, 0.1, 1$). We plot the average log-value estimates (\emph{left}), the L2 gradient norm in log-domain (\emph{center}), and the average return (\emph{right}) for 20 episodes over 200k environment steps. For $\beta < 0$, we observe that the value estimates and the gradient norm are almost zero, leading to unstable learning. For $\beta > 0$, the value estimates and gradients are orders of magnitude larger than DQN's estimates. These estimates are propagated by the gradients, resulting in exploding gradients. 
        }\label{fig:unstable}
\end{figure*}

Motivated by theoretical results established in Sec.~\ref{sec:theoretical}, we now derive a risk-sensitive off-policy actor-critic algorithm that avoids sampling the entire trajectory. We begin by illustrating the issues that arise when learning a critic network that approximates the exponential value function, using the squared exponential temporal-difference (TD). We then address these problems by introducing a critic representation that avoids explicit computation of the exponential value function, and design a stabilizing mechanism for the critic gradient, based on batch normalization and clipping. Finally, we combine our improvements 
to propose: risk-sensitive exponential actor-critic (rsEAC), a practical algorithm that considerably improves the stability of the optimization of the entropic risk measure.

\subsection{Numerical Issues of Exponential TD Learning}
As discussed in Sec.~\ref{sec:entropic_risk}, sample-based estimates of $Q_{\mu_\theta}^\beta$ are biased due to nonlinear Bellman equations. To obtain unbiased estimates, we must consider the exponential value functions:
\begin{gather}\label{eq:exponential_functions}
Z_{\mu_\theta}^\beta(s_t) = e^{\beta V_{\mu_\theta}(s_t)}, \quad Z_{\mu_\theta}^\beta(s_t, a_t) = e^{\beta Q_{\mu_\theta}(s_t,a_t)}.
\end{gather}
These functions are related by the exponential Bellman equations, which are obtained by applying an exponential transformation to both sides of \EQN\eqref{eq:bellman_Q}:
\begin{gather}\label{eq:soft-bellmanV}
 Z_{\mu_\theta}^\beta(s_t) = \mathbb{E}_{\mu_\theta(a_t|s_t)} [Z_{\mu_\theta}^\beta(s_t,a_t)], \quad  \\ Z_{\mu_\theta}^\beta(s_t,a_t) = e^{\beta r_t}  \mathbb{E}_{p(s_{t+1} \mid s_t, a_t)} \left[Z_{\mu_\theta}^\beta(s_{t+1})\right].
\end{gather}
From these equations, $Z_{\mu_\theta}^\beta$ can be estimated using only samples, which can then be transformed to recover $Q_{\mu_\theta}^\beta$ \citep{fei2021exponential}. However, it is impractical to learn different exponential value functions for each state and action pair, especially in environments with large state/action spaces. Instead, these functions can be approximated with
a parameterized critic network $Z_\psi(s_t,a_t)$, which is trained using stochastic gradient descent to minimize the squared exponential TD error: 
\begin{equation}\label{Z_loss}
\mathop{\mathbb{E}}_{(s_t, a_t, r_t, s_{t+1}) \sim \mathcal{D}} \left[\left(Z_\psi(s_t,a_t) - e^{\beta r_t} Z_{\psi'}(s_{t+1})\right)^2\right],
\end{equation}
where $\Dcal$ is an experience replay buffer that stores previously seen interactions with the environment. The value function $Z_{\psi'}(s_{t+1})$ is implicitly parameterized with a target critic network as $\max_{a_{t+1}}Z_{\psi'}(s_{t+1}, a_{t+1})$. Risk-sensitive actor critic (R-AC) \citep{noorani2022risk, noorani2023exponential} considers a similar loss in the context of online learning for its critic. However, two major drawbacks persist in the learning the R-AC critic network: first, the value function estimates the exponentiated reward, $\mathbb{E}[e^{\beta \sum_t r_t}]$, which can lead to numerical overflow / underflow of the network weights.  Second, directly minimizing \EQN\eqref{Z_loss} may cause exploding (or vanishing) gradients, resulting in divergent behavior. \FIG \ref{fig:unstable}, demonstrates these instabilities by learning value functions in the CartPole task \citep{barto1983neuronlike} in Gymnasium \citep{towers_gymnasium_2023}.



\begin{figure*}[t]
    \centering
     \begin{subfigure}[b]{0.26\textwidth}
         \includegraphics[width=\textwidth]{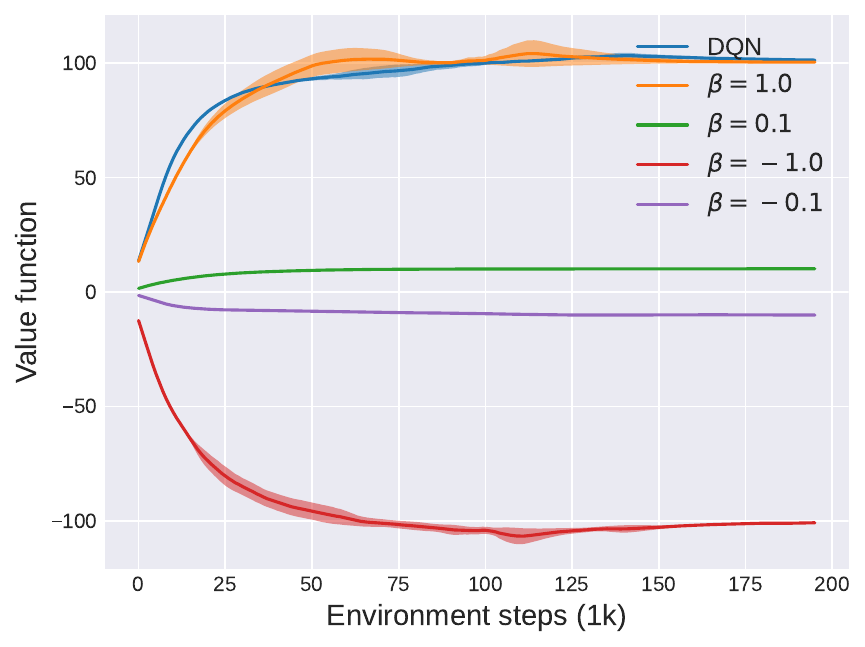}
         \caption{Average Initial Value}
         \label{fig:stable_value}
     \end{subfigure} 
    \hspace{-10px}
     \begin{subfigure}[b]{0.26\textwidth}
         \includegraphics[width=\textwidth]{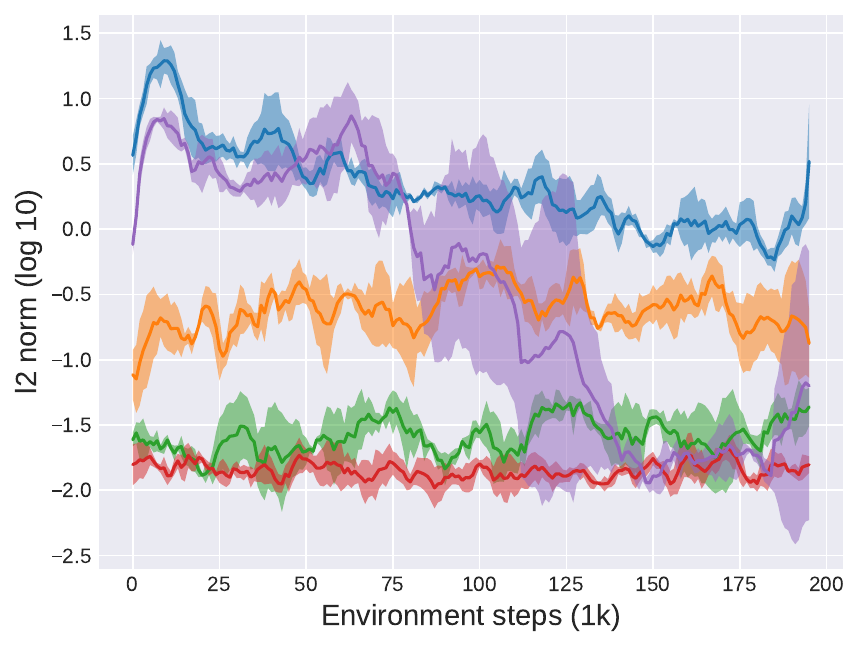}
         \caption{L2 Gradient Norm}
         \label{fig:stable_gradient}
     \end{subfigure}
        \hspace{-10px}
     \begin{subfigure}[b]{0.26\textwidth}
         \includegraphics[width=\textwidth]{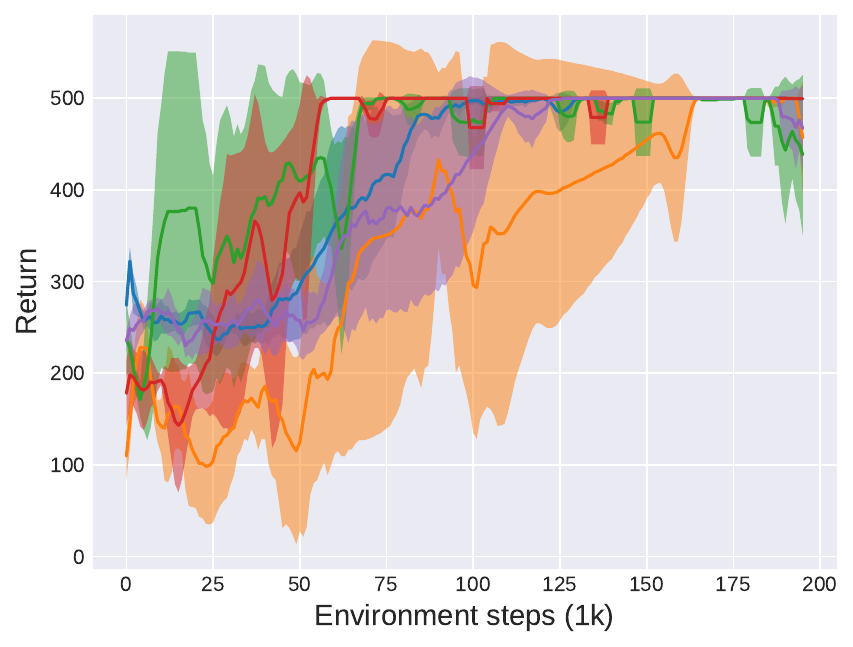}
         \caption{Average Performance}
         \label{fig:stable_performance}
     \end{subfigure}
     \caption{\textbf{Stable value functions and gradients.} We train value functions for DQN and $Q_\psi(s_t,a_t)$ using the normalized clipped gradient for a range of $\beta$ settings ($-1, -0.1, 0.1, 1$). We plot the average value estimates (\emph{left}), the L2 gradient norm in log-domain (\emph{center}), and the average return (\emph{right}) for 20 episodes over 200k environment steps. }\label{fig:stable}
\end{figure*}

\subsection{Stabilizing Updates}\label{subs:instabilities}
We now consider a critic representation that avoids numerical instabilities by parameterizing the exponential value function as $Z_\psi(s,a) = e^{Q_\psi(s,a)}$ where $Q_\psi(s,a)$ is a neural network. In particular, $\frac{1}{\beta}Q_\psi(s,a)$ serves as our approximation of the soft-value function $Q_{\mu_\theta}^\beta$. This reparameterization is more numerically stable as gradients $\nabla Q_\psi$ are taken in the log-domain.  We train this function by minimizing the squared exponential TD error $J_Q(\psi)$:
\begin{equation}\label{eq:reparam}
\!\!\!\mathbb{E}_{(s_t, a_t, r_t, s_{t+1}) \sim \mathcal{D}}\left[\left(e^{Q_\psi(s_t,a_t)} - e^{\beta r_t+Q_{\psi'}(s_{t+1})}\right)^2 \right].
\end{equation}
where $\mathcal{D}$ is the experience replay buffer. We minimize \EQN\eqref{eq:reparam} using stochastic gradient descent. The value function $Q_{\psi'}(s_{t+1})$ is also implicitly parameterized with a target critic network as $\max_{a_{t+1}} Q_{\psi'}(s_{t+1},a_{t+1})$.  The gradient $\nabla J_Q(\psi)$ is then:
\begin{equation}
 \!\!\!e^{Q_\psi(s_t,a_t)}\!\left(e^{Q_\psi(s_t,a_t)} \!\!- e^{\beta r_t +Q_{\psi'} (s_{t+1})}\right) \nabla Q_\psi(s_t,a_t).
\end{equation}
Note that the trailing gradient $\nabla Q_\psi$ is taken in log-domain, and so is numerically stable.  Ignoring this term the leading factors are of the form: $e^x(e^x - e^y)$, which can be numerically unstable.  First, we rearrange factors so that the numerical instability occurs only in the leading term.  To do this we factor out $e^x$ or $e^y$, whichever is larger.  We define the helper function,
\begin{equation}
 f(x,y) \triangleq 
    \begin{cases}
        (1-e^{y-x}), & \text{if } x\geq y\\
        (e^{x-y}-1), & \text{otherwise.}
    \end{cases}
\end{equation}
Then we have that \mbox{$ e^x(e^x - e^y) = e^{x + \max(x, y)} f(x,y)$}.  Observe that the trailing term $f(x,y)$ is numerically stable as it is bounded: $f(x,y)\in[-1,1]$.  Numerical instability occurs only in the leading term $e^{x+\max(x,y)}$, which must be stabilized by subtracting a constant $z$ in log-domain, which is equivalent to normalizing \mbox{$e^x(e^x - e^y) \propto e^{x + \max(x, y)-z} f(x,y)$}. We defer the details of our normalizing choice $z$ and our clipping mechanism to Appendix B.

\textbf{Evaluation} We now show that using the normalized clipped gradient addresses the numerical issues of $Z_\psi(s_t,a_t)$ and can stabilize the learning of the exponential-TD objective. Again we consider the CartPole environment, and train the value functions $Q_\psi(s_t,a_t)$ for the same set of $\beta$ values as in \FIG\ref{fig:unstable}. In \FIG\ref{fig:stable_value}, we observe that the value function $Q_\psi(s_t, a_t)$ is less prone to suffer from numerical overflow compared to $Z_\psi(s_t, a_t)$. In fact, the magnitude learned by DQN \citep{mnih2015human} is close to $\frac{1}{\beta} Q_\psi(s_t, a_t)$, which is not surprising given that CartPole is a deterministic task and our approach scales the rewards by $\beta$. Using the normalized clipped gradient, our approach also guarantees gradients with more manageable magnitudes compared to directly optimizing the exponential-TD objective (\FIG\ref{fig:stable_gradient}). Finally, our algorithm is capable of learning optimal policies for the all the tested $\beta$ initializations, which was only possible for one setting when learning $Z_\psi(s_t, a_t)$ (See \FIG\ref{fig:stable_performance}).

\subsection{Actor-Critic Optimization}

We now present risk-sensitive exponential actor-critic (rsEAC), a practical off-policy algorithm that optimizes the entropic risk-measure. We develop an actor-critic algorithm that updates its policy in the direction of the off-policy deterministic policy gradient in \EQN\eqref{eq:risk_off_gradient}, and substitutes the soft-value function $Q_{\mu_\theta}^\beta$ with our approximation $Q_\psi$. We build on the twin delayed deep deterministic policy gradient algorithm (TD3) \citep{fujimoto2018addressing} for its mechanisms on over-estimation bias reduction and its empirical success in continuous control tasks. Similarly to TD3, we maintain a pair of critics with parameters $\psi_1$ and $\psi_2$, a single actor with parameter $\theta$, and their respective target networks $\psi_1'$, $\psi_2'$ and $\theta'$. For each timestep, we update the pair of critic networks by replacing the target in the normalized clipped gradient derived in Section \ref{subs:instabilities} with 
\begin{equation}
\!\!\!y_t \!= 
    \begin{cases}
       \beta r_t + \min_{i=1,2} \gamma Q_{\psi_i'}(s_{t+1}, a_{t+1}),& \text{if } \beta > 0\\
         \beta r_t + \max_{i=1,2} \gamma Q_{\psi_i'}(s_{t+1}, a_{t+1}), & \text{if } \beta < 0.
    \end{cases}
\end{equation}
where $a_{t+1} \!= \!\mu_{\theta'}(s_{t+1}) \!+ \!\epsilon$ and $\epsilon \sim \text{clip}(\mathcal{N}(0, \sigma^2), -c^*,c^*)$. The actor is a deterministic policy parameterized with a neural network $\mu_\theta(s_t)$ and is optimized w.r.t.~the risk-sensitive off-policy gradient in \EQN\eqref{eq:risk_off_gradient}, for which we substitute the soft-value function with our estimate $\frac{1}{\beta}Q_{\psi_1}(s_t,a_t)$:
\begin{equation}
\mathbb{E}_{s_t \sim \mathcal{D}}\left[ \frac{1}{\beta}  \nabla_{\theta} \mu_\theta(s_t) \nabla_{a_t} Q_{\psi_1}(s_t, a_t) \mid_{a_t = \mu_\theta(s_t)} \right].
\end{equation}
where $s_t$ is sampled from the buffer $\mathcal{D}$. Given that the policy $\mu_\theta(s_t)$ is deterministic, we ensure adequate exploration by adding Gaussian noise $\mathcal{N}(0,\sigma^2)$. We update the target networks using an exponentially moving average of the weights. Pseudocode for rsEAC can be found in Appendix D.  

\section{Related Work}
The entropic risk measure was first investigated in \citep{howard1972risk} as an approach to incorporate risk sensitivity in MDPs. Following this work, risk-sensitive MDPs --- where the agent knows the transition dynamics or has access to a simulator of the environment --- have been widely studied \citep{fleming1995risk, hernandez1996risk, coraluppi1999risk, di1999risk, borkar2002risk, huang2020stochastic}. More recently, risk-sensitive RL has been connected to the RL-as-inference framework \citep{noorani2022probabilistic}, in which policy search is formulated as maximum likelihood estimation on an augmented MDP \citep{todorov2008general, kappen2012optimal,levine2018reinforcement}. Another formulation of risk-sensitive RL with probabilistic inference is that of policy search as an expectation maximization (EM) style algorithm~\citep{neumann2011variational, abdolmaleki2018maximum, ijcai2021p316, Granados2025model}. 


Risk-sensitive model-free algorithms optimize the exponential criteria as an equivalent objective of the entropic risk measure \citep{bauerle2014more, borkar2001sensitivity, borkar2002q}. These approaches employ exponential value functions --- functions that are associated in a multiplicative way, rather than in an additive way as in standard RL \citep{fei2021exponential} --- that can be estimated using Monte Carlo methods. However, the exponential criteria suffer from numerical instabilities that arise from the estimation of the expected exponentiated return. These numerical instabilities are further exacerbated when introducing function approximators \citep{maei2009convergent}. Risk-sensitive policy gradient algorithms \citep{nass2019entropic, noorani2021risk} may suffer from these same instabilities. 

Network saturation for the exponential function is a well-known issue in neural network design \citep{goodfellow2016deep}. While log-likelihood objectives can prevent this problem by undoing the exponential operation, this approach is incompatible with exponential TD learning, which must use the squared loss. Another popular approach for training models with saturating nonlinearities is batch normalization which normalizes each layer of a deep neural network \citep{ioffe2015batch}. Motivated by this approach, we design a batch normalization scheme for the last layer to avoid the saturation of the gradient, and introduce gradient clipping on the exponent, which helps with the exploding/vanishing gradient problem \citep{pascanu2013difficulty}. Our method is, to our knowledge, the first actor-critic algorithm that incorporates mechanisms to stabilize the optimization of the exponential TD learning, and our policy gradient theorems establish a rigorous framework for risk-sensitive actor-critic learning with the entropic risk measure.
\section{Experiments}

\begin{figure*}[t]
    \centering
     \begin{subfigure}[b]{0.225\textwidth}
         \includegraphics[width=\textwidth]{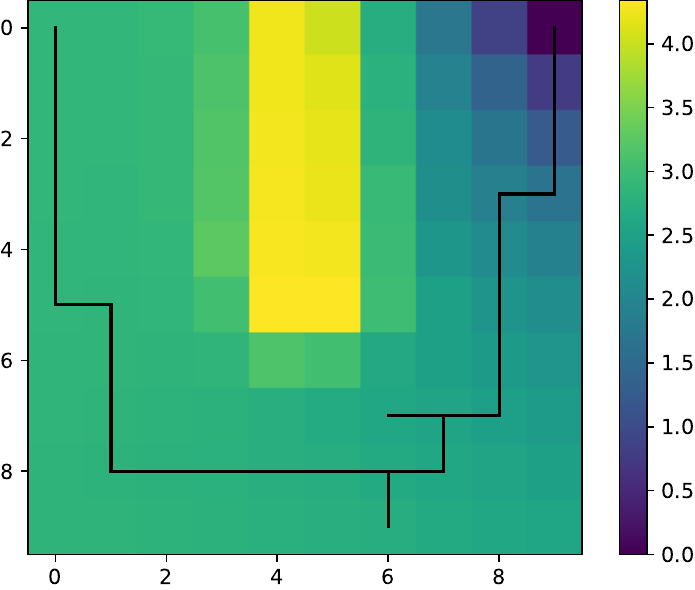}
         \caption{Risk-averse ($\beta = -1$)}
         \label{fig:-1}
     \end{subfigure}
     \begin{subfigure}[b]{0.23\textwidth}
         \includegraphics[width=\textwidth]{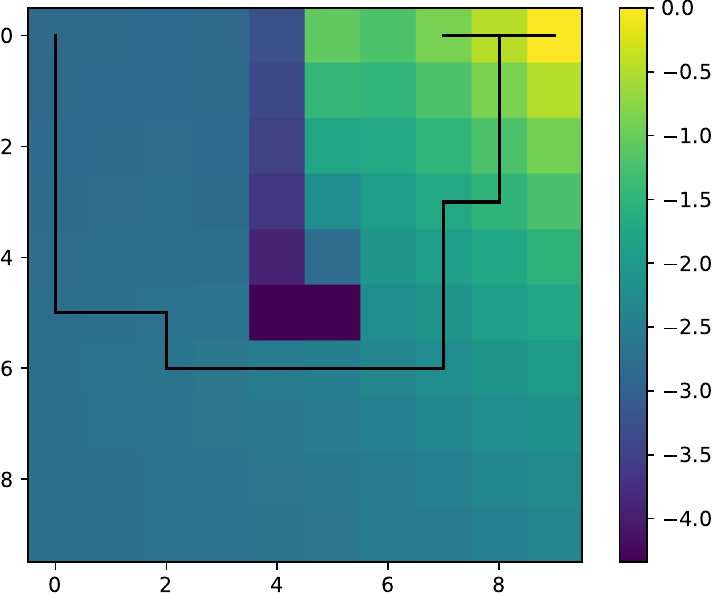}
         \caption{Risk-Seeking ($\beta = 1$)}
         \label{fig:1}
     \end{subfigure} 
     \begin{subfigure}[b]{0.235\textwidth}
         \includegraphics[width=\textwidth]{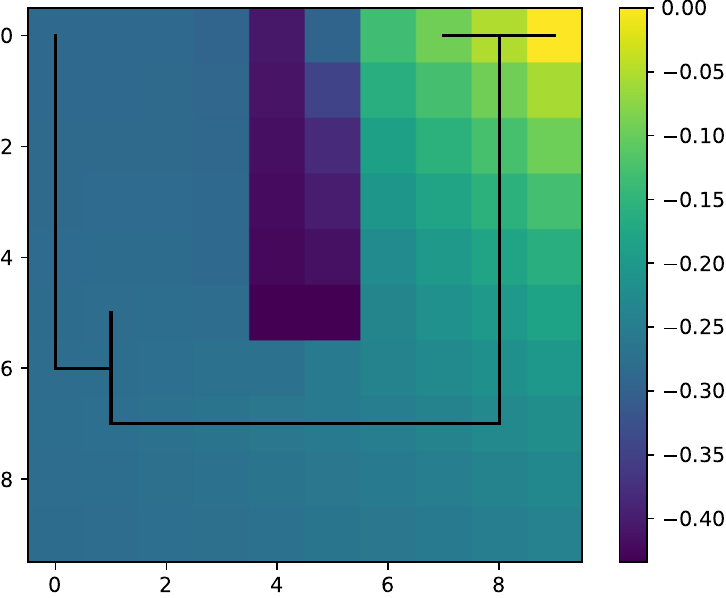}
         \caption{Risk-neutral ($\beta = 0.1$)}
         \label{fig:0.1}
     \end{subfigure} 
    \caption{\textbf{Tabular Risk-Sensitive Policies.}  Optimal value functions $Z^*(s_t)$ (in log-domain) and corresponding trajectory for different $\beta$ values. Agents trained with negative $\beta$ values learn risk-averse policies that avoid the cliff-region, while agents trained with positive $\beta$ values tend to be risk-seeking and hug the cliff closely. We also recover risk-neutral policies when the magnitude of $|\beta|$ is small.}\label{fig:gridworld_risky}
    \vspace{-5mm}
\end{figure*}

\begin{figure}
    \centering
     \begin{subfigure}[b]{0.16\textwidth}
         \includegraphics[width=\textwidth]{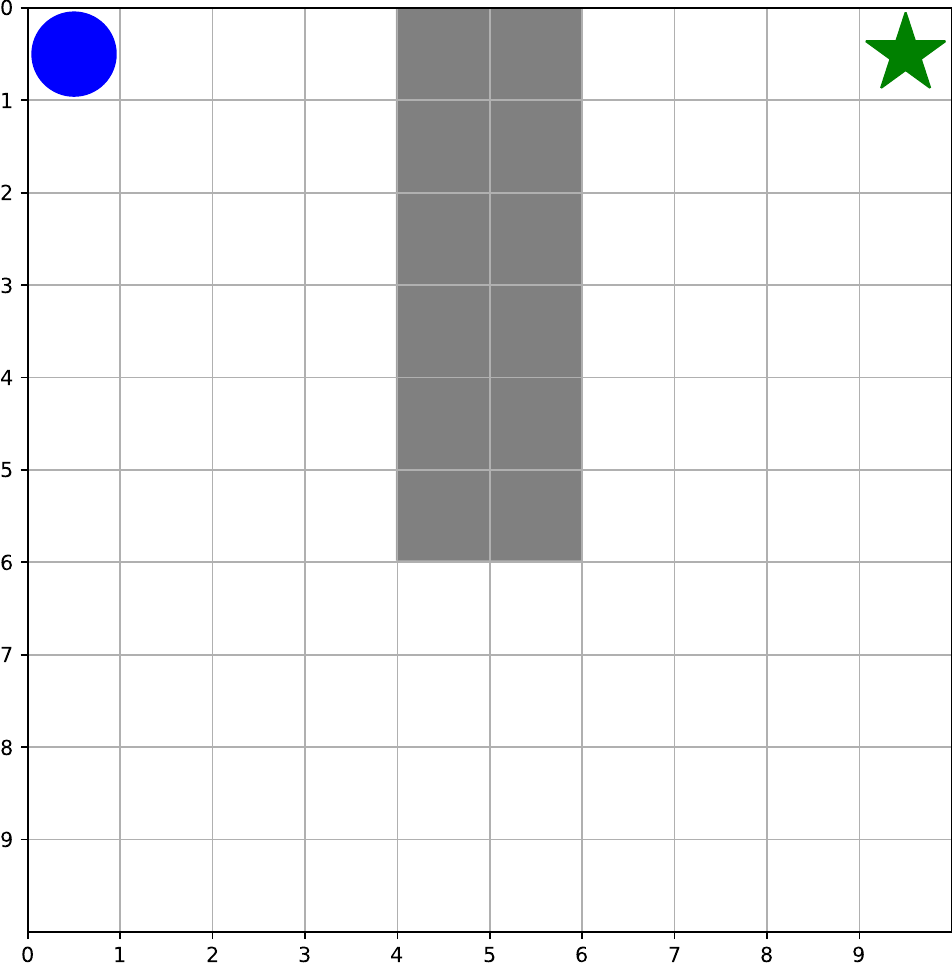}
     \end{subfigure}
     \begin{subfigure}[b]{0.19\textwidth}
         \includegraphics[width=\textwidth]{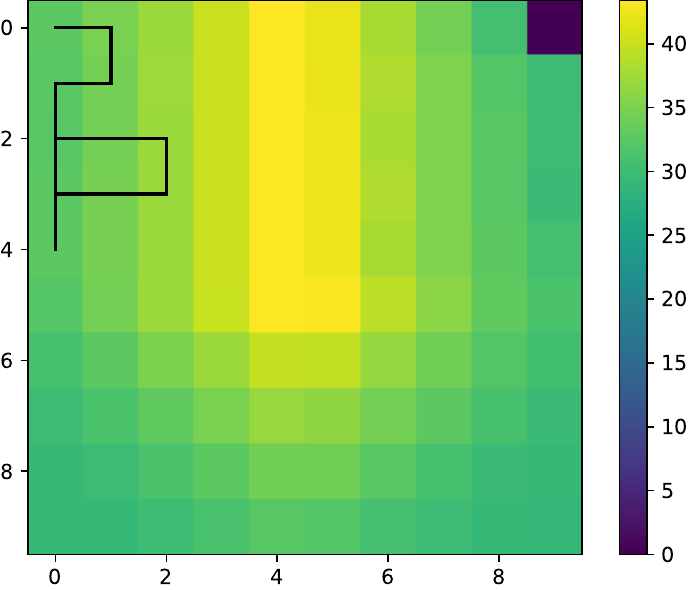}
     \end{subfigure} 
    \caption{\textbf{Stochastic Cliff GridWorld.} \emph{Left}: 2D grid environment with initial-state and goal-state given by blue circle and green star, respectively. The cliff region is given by gray states. \emph{Right}: We plot the learned optimal value functions (in log-domain) for $\beta = -10$ and a sampled trajectory. The estimated values explode in magnitude, resulting in a policy incapable of reaching the destination. }\label{fig:gridworld}
\end{figure}

The goals of our experiments are to: verify that numerical instabilities are a direct result of optimizing the exponential Bellman equations, study the risk preference and robustness of rsEAC to different values of $\beta$, and demonstrate that our approach can produce competitive risk-sensitive policies in comparison to other risk-sensitive methods.  We begin by considering a tabular environment where we can study the objective without introducing additional approximation error. We then demonstrate that rsEAC can learn stable policies while varying the risk parameter $\beta$.  Finally, we compare rsEAC to risk-averse baselines in three challenging risky variations of MuJoCo tasks, where risk-aware agents should avoid a noisy region.

\subsection{GridWorld Environment}


Our aim with this experiment is to illustrate the role that $\beta$ plays in modulating risk-sensitivity and identify potential numerical instabilities that may appear when optimizing the exponential Bellman equations in the tabular setting. We examine the 2D GridWorld described in~\citep{eysenbach2022mismatched}. In this environment, the agent starts in an initial position and its objective is to travel to a given destination. When the agent selects a direction (up, left, down and right), it may possibly move to its chosen direction or to a random neighboring state. We introduce aleatoric risk by defining a cliff-region that the agent should avoid. Falling into the cliff incurs a large negative reward and terminates the episode. 

The GridWorld environment is composed of a 2D grid with 10 × 10 states (see \FIG\ref{fig:gridworld}). The initial state and destination are depicted by a blue circle and green star, respectively. The cliff-region is depicted by a grey rectangle. The probability of moving to a random neighboring state is 0.2. Each move has a $-1$ cost and the penalty for falling into the cliff is $-10$. We learn the value functions $Z^*(s_t,a_t)$ using the exponential Bellman equations, with an $\epsilon$-greedy policy ($\epsilon = 0.1$) and a discount factor of $\gamma = 0.85$. We train this algorithm by running it for 50000 episodes.

In \FIG\ref{fig:gridworld_risky}, we plot learned optimal value functions $Z^*(s_t)$ (in log-domain) and a sampled trajectory using a learned optimal policy $\pi^*$ for a range of $\beta$ values. We observe that policies trained with $\beta < 0$ prefer longer paths that avoid the cliff-region (\FIG\ref{fig:-1}), while policies trained with $\beta > 0$ choose the shortest path which travels next to the cliff (\FIG\ref{fig:1}). As expected, we recover risk-neutral policies when the magnitude of $|\beta|$ is small (\FIG\ref{fig:0.1}). The learned value functions also demonstrate exploding/vanishing gradient magnitudes when estimated using $|\beta|$ values with larger magnitudes (\FIG\ref{fig:gridworld}). We expect these numerical instabilities to develop even more often when introducing function approximators.


\begin{figure}
    \centering
    \hspace{-15px}
     \begin{subfigure}[b]{0.24\textwidth}
         \includegraphics[width=\textwidth]{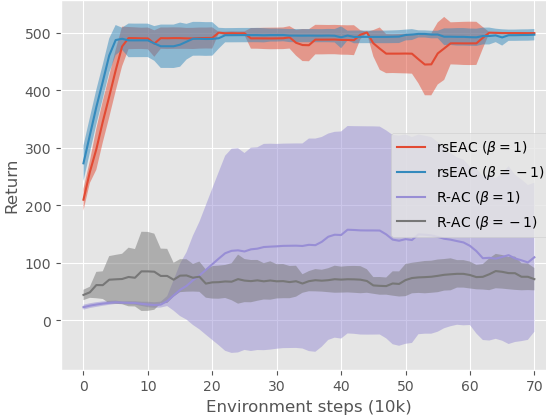}
         \label{fig:Inverted_perf}
     \end{subfigure} 
     \hspace{-5px}
     \begin{subfigure}[b]{0.24\textwidth}
         \includegraphics[width=\textwidth]{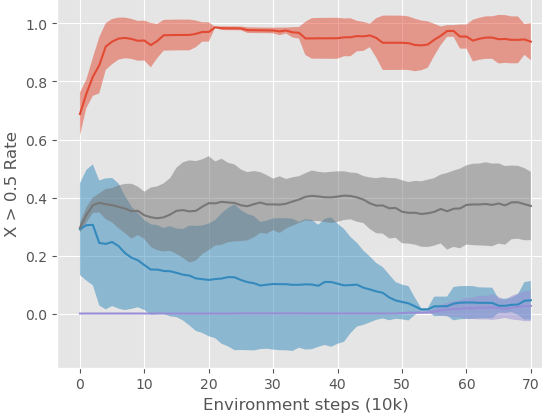}
         \label{fig:Inverted_perc}
     \end{subfigure}
     \caption{\textbf{Risk-sensitive policies in Inverted Pendulum.} We train rsEAC and R-AC for $\beta =1 $ and $\beta = -1$. \emph{Left}: We plot the average return. \emph{Right}: Percentage of steps on an episode in risky regions. The solid curves correspond to the mean and shaded regions to $\pm$ one standard deviation over 5 random seeds. rsEAC learns, both, risk-seeking and risk-averse policies that achieve high return, while R-AC encounters numerical instabilities and learns poor policies.}
     \label{fig:ablation}
\end{figure}

\begin{figure*}
    \centering
     \begin{subfigure}[b]{0.26\textwidth}
         \includegraphics[width=\textwidth]{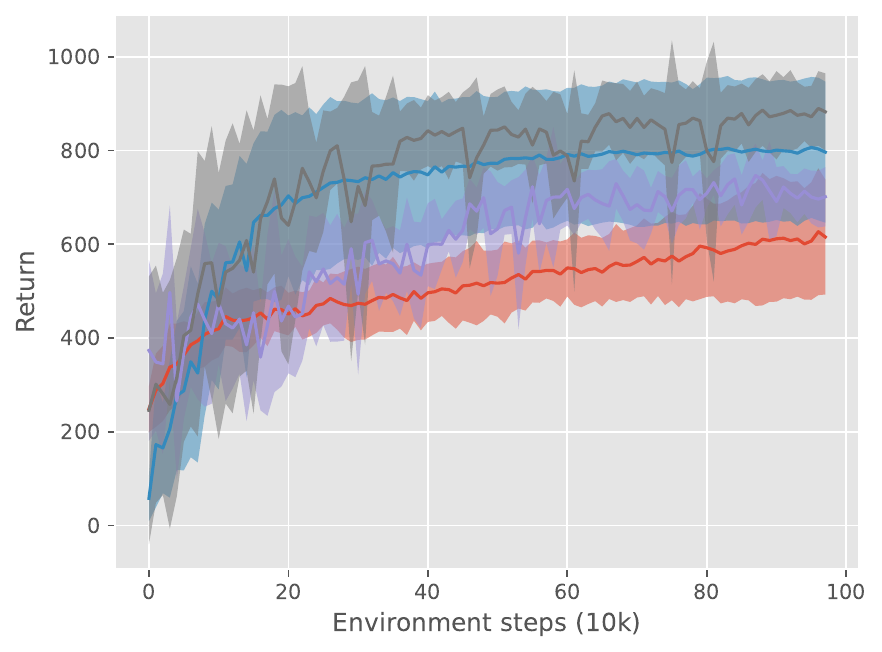}
         \caption{Swimmer-v4}
         \label{fig:Swim_perf}
     \end{subfigure} 
     \begin{subfigure}[b]{0.26\textwidth}
         \includegraphics[width=\textwidth]{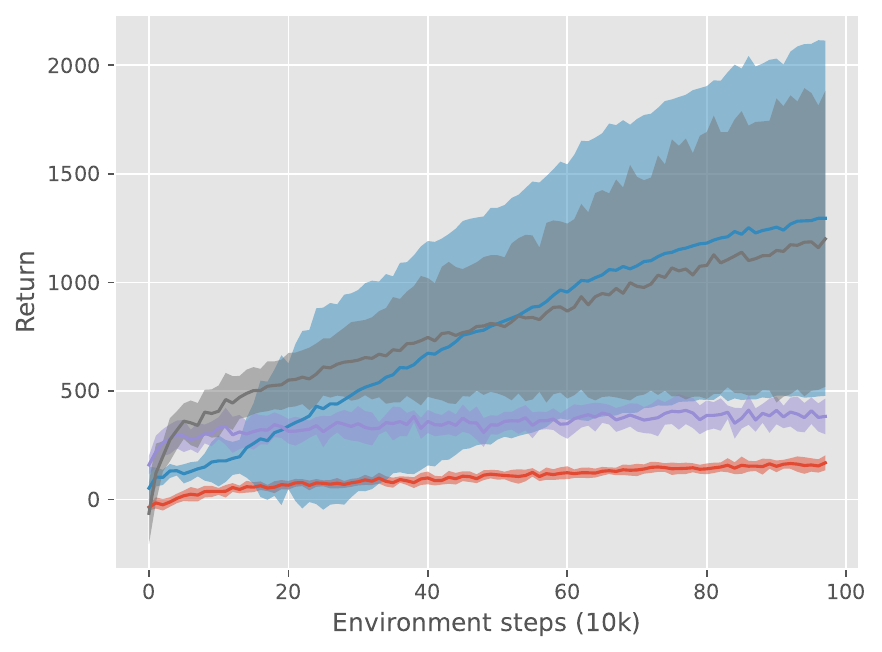}
         \caption{HalfCheetah-v4}
         \label{fig:Half_perf}
     \end{subfigure}
     \begin{subfigure}[b]{0.26\textwidth}
         \includegraphics[width=\textwidth]{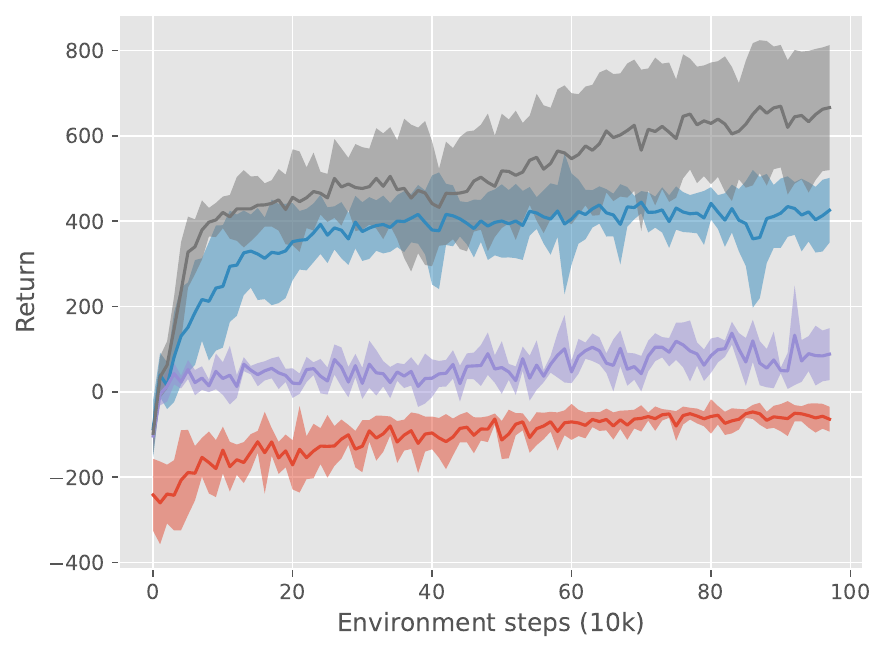}
         \caption{Ant-v4}
         \label{fig:Ant_perf}
     \end{subfigure}
    \centering
     \begin{subfigure}[b]{0.26\textwidth}
         \includegraphics[width=\textwidth]{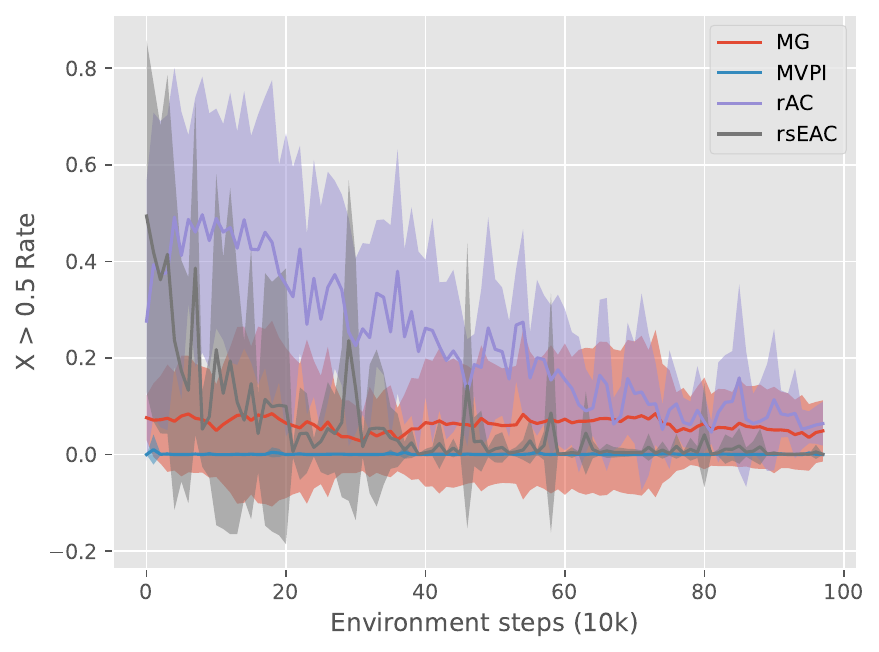}
         \label{fig:Swim_perc}
     \end{subfigure} 
     \begin{subfigure}[b]{0.26\textwidth}
         \includegraphics[width=\textwidth]{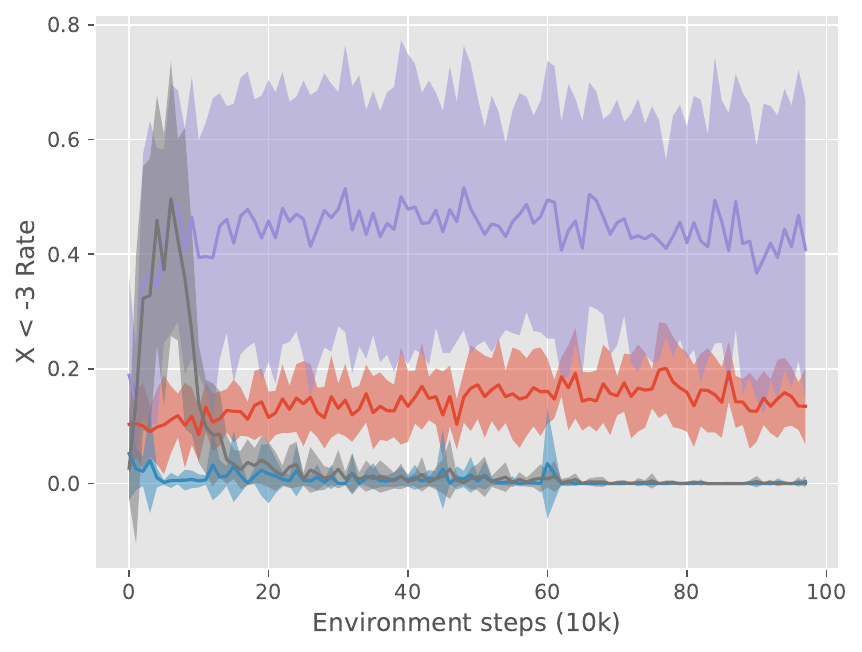}
         \label{fig:Half_perc}
     \end{subfigure}
     \begin{subfigure}[b]{0.26\textwidth}
         \includegraphics[width=\textwidth]{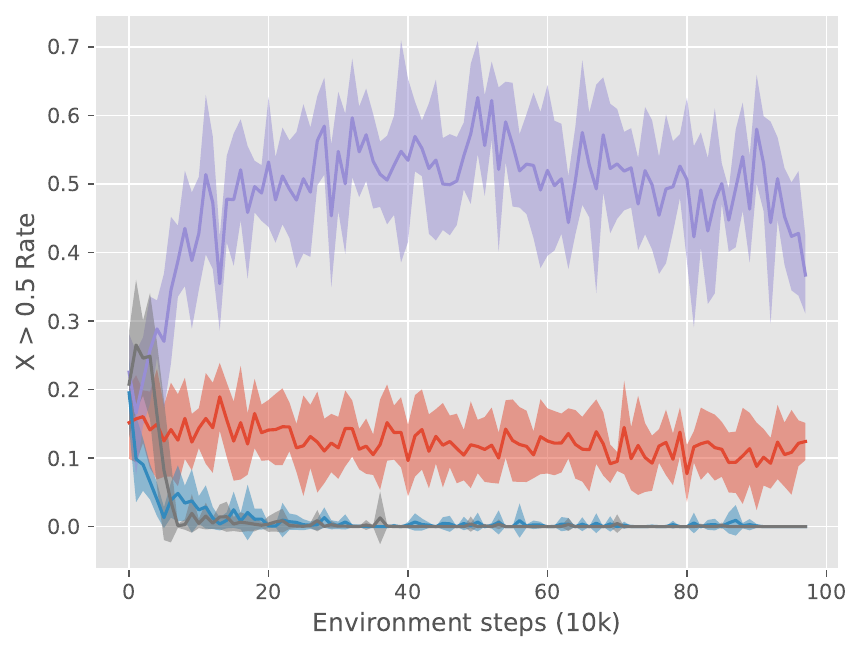}
         \label{fig:Ant_perc}
     \end{subfigure}
     \caption{\textbf{Risk-Averse MuJoCo.} \emph{Top row}: Average return on MuJoCo benchmarks with risky regions. \emph{Bottom row}: Percentage of steps on an episode in risky regions. The solid curves correspond to the mean and shaded regions to $\pm$ one standard deviation over 10 random seeds. 
        }\label{fig:mujoco_reward}
\end{figure*}

\subsection{Inverted Pendulum Environment}
We study the role of $\beta$ in modulating risk in rsEAC and that stable policies can be learned for different risk profiles. We consider a risky variant of Inverted Pendulum --- a continuous version of the CartPole environment --- from MuJoCo \citep{todorov2012mujoco}. Following \citep{luo2023alternative}, we induce aleatoric risk by adding stochastic noise ($\mathcal{N}(0,1)$) to the reward function when the agent's X-position is greater than 0.01.  For our evaluations, we consider R-AC \citep{noorani2023exponential} and our method using two different $\beta$ values corresponding to risk-seeking ($\beta = 1$) and risk-averse ($\beta = -1$) settings. In \FIG\ref{fig:ablation}, we plot the average return performance and the risky region visiting rate. We report the average and standard deviation for 5 independents runs. We verify that rsEAC can learn risk-seeking and risk-averse policies --- our agent prefers the risky region when trained with $\beta = 1$, and avoids it when trained with $\beta = -1$ --- while retaining high performance in terms of average return. In comparison, R-AC suffers from numerical instabilities that lead to low-return policies and a lack of risk-sensitivity.

\subsection{Simulated Robotic Benchmarks}

We evaluate rsEAC on three modified MuJoCo \citep{todorov2012mujoco} continuous control tasks: Swimmer, HalfCheetah, and Ant. Following \citep{luo2023alternative}, we modify their reward functions to make the speed positive in both directions so the agent is free to move left or right. We also construct risky regions in these tasks based on the X-position. We add a stochastic reward sampled from $\mathcal{N}(0,10^2)$ for Swimmer and HalfCheetah, when the agent's position is greater than 0.5 and smaller than -3, respectively. For the Ant task, we add a stochastic reward sampled from $\mathcal{N}(0,7^2)$ when its position is greater than 0.5. For all tasks, we append the X-position into the agent's observation. Hence, a risk-averse agent would do its best effort to avoid these risky regions. 

We compare our algorithm against three risk-averse baselines: Mean Gini deviation (MG) \citep{luo2023alternative}, a policy gradient algorithm that considers Gini deviation as its risk measure and outperforms other return variance methods; mean-variance policy iteration (MVPI) \citep{zhang2021mean}, a flexible algorithm that optimizes reward-volatility risk measure and demonstrates great empirical success in risk-averse continuous domains; and risk-sensitive actor critic (R-AC) \citep{noorani2023exponential}, an online actor-critic algorithm that optimizes the entropic risk measure. To make the comparison fair between different methods, we implement every actor-critic algorithm on top of TD3 \citep{fujimoto2018addressing}, and use the same network architecture across all algorithms. For MG, we use PPO \citep{schulman2017proximal} to learn its policy gradient, as suggested by the author's implementation. See Appendix C for additional details.

We report the average return (\emph{top-row}) and the risky region visiting rate (\emph{bottom-row}) for the three risk-averse MuJoCo environments (See \FIG\ref{fig:mujoco_reward}). The results show that rsEAC produces policies that are risk-averse (low rate in risky regions), and perform comparably to other baselines in terms of average return. In particular, it outperforms R-AC in every task, showing that our stabilizing mechanisms play an essential role in the learning of risk-averse and high-return policies. Our approach performs comparably to MVPI in terms of final performance across environments, while outperforming it in the higher-dimensional Ant task.

\section{Conclusion}
In this paper, we study risk-sensitive reinforcement learning using the entropic risk-measure objective. We establish that existing policy gradient approaches learn high-variance and numerically unstable estimates that may lead to divergent behavior by the agent. To resolve these issues, we derive policy gradient theorems that circumvent the need to estimate a gradient for the full trajectory, and propose a critic parameterization for the exponential value function that avoids its explicit computation. We optimize this critic using stochastic gradient descent, for which we too propose a stabilizing mechanism based on mini-batch normalization and clipping. Taking these components together, we then propose risk-sensitive exponential actor-critic (rsEAC), a practical algorithm that greatly improves the stability of the optimization of the entropic risk measure. Our evaluations demonstrate that rsEAC can learn risk-sensitive and high-return policies in complex continuous tasks.

\textbf{Limitations} Our approach keeps two main limitations. There is an inherent instability when working with the exponential function, and even our solutions may not be enough to stabilize updates in all settings. Second, a key component for our method is the risk parameter $\beta$ which must be tuned to balance the amount of risk-sensitivy for different tasks.

\bibliography{refs.bib}
\newpage
\appendix
\section{Proofs of Theoretical Results}\label{sec:proofs}
\subsection{Proof of Risk-Sensitive Stochastic Policy Gradient Theorem}
This proof considers similar regularity assumptions as in \cite{silver2014deterministic}:

\paragraph{Assumption 1:} $\pi_\theta(a|s)$, $\nabla_\theta \pi_\theta(a|s)$, $p(s'|s, a)$, $p_1(s)$ are continuous in $s$, $a$, $s'$ and $\theta$.

\paragraph{Assumption 2:} there exists a $b$ such that $\sup_s p_1(s) < b$, $\sup_{s,a,s'} p(s'|s, a) < b$, and $\sup_{s,a} r(s, a) < b$.

\paragraph{Assumption 3:} $\mathcal{A}$ is a compact subset of $\mathbb{R}^m$ and $\mathcal{S}$ is a compact subset of $\mathbb{R}^d$.

In particular, assumption 1 implies that $V_{\pi_\theta}^\beta(s)$, $\nabla_\theta V_{\pi_\theta}^\beta(s)$, $Q_{\pi_\theta}^\beta(s,a)$ and $\nabla_\theta Q_{\pi_\theta}^\beta(s,a)$ are continuous functions of $\theta$, $s$, and $a$; and the compactness of $\mathcal{S}$ and $\mathcal{A}$ further implies that for any $\theta$, $||\nabla_\theta V_{\pi_\theta}^\beta(s)||$, and $||\nabla_\theta \pi_\theta(a|s)||$ are bounded functions of $s$ and $a$. In addition, we define $p_1^*(s)$ as the initial distribution proportional to $p_1(s) e^{\beta V_{\pi_\theta}^\beta(s)}$, the distribution of $s'$ after transitioning for $t$ time steps from state $s$ using dynamics $p^*(s'|s,a)$ and policy $\pi_\theta^*(a|s)$ as $p^*(s\to s', t, \pi^*)$, and $\rho_\pi^*(s') := \int_{\mathcal{S}} \sum_{t=0}^T  p_1^*(s) p^*(s \to s', t, \pi^*) \text{d}s$. Now we present our result,

\renewcommand\thetheorem{3.1}
\begin{theorem}
(Risk-Sensitive Stochastic Policy Gradient Theorem). The gradient w.r.t. $J^\beta(\pi_\theta)$ is given by,
\begin{gather*}
    \nabla_\theta J^\beta(\pi_\theta) = \frac{1}{\beta} \int_\mathcal{S} \rho_{\pi}^*(s) \int_\mathcal{A} \nabla_\theta \pi_\theta (a|s)\frac{e^{\beta Q_{\pi_\theta}^\beta(s,a)}}{e^{\beta V_{\pi_\theta}^\beta(s)}} \,\mathrm{d}a \,\mathrm{d}s.
\end{gather*}
\end{theorem}

\begin{proof}
\begin{equation*}
    \begin{aligned}
    \nabla_\theta V_{\pi_\theta}^\beta(s) 
    &= \nabla_\theta \frac{1}{\beta}\log \int_\mathcal{A} \pi_\theta(a|s) e^{\beta Q_{\pi_\theta}^\beta(s, a)} \text{d}a \\
    &= \frac{1}{\beta e^{\beta V_{\pi_\theta}^\beta(s)}} \nabla_\theta \int_\mathcal{A} \pi_\theta(a|s) e^{\beta Q_{\pi_\theta}^\beta(s, a)}\text{d}a \\
    &\overset{(a)}{=} \frac{1}{\beta} \int_\mathcal{A} \nabla_\theta \pi_\theta(a|s) \frac{e^{\beta Q_{\pi_\theta}^\beta(s, a)}}{e^{\beta V_{\pi_\theta}^\beta(s)}} \text{d}a \\ &+ \int_\mathcal{A} \frac{\pi_\theta(a|s) e^{\beta Q_{\pi_\theta}^\beta(s, a)}}{ e^{\beta V_{\pi_\theta}^\beta(s)}} \nabla_\theta Q_{\pi_\theta}^\beta(s, a) \text{d}a \\
    &= \frac{1}{\beta} \int_\mathcal{A} \nabla_\theta \pi_\theta(a|s) \frac{e^{\beta Q_{\pi_\theta}^\beta(s, a)}}{e^{\beta V_{\pi_\theta}^\beta(s)}} \text{d}a \\ &+ \int_\mathcal{A} \frac{\pi_\theta(a|s) e^{\beta Q_{\pi_\theta}^\beta(s, a)}}{e^{\beta V_{\pi_\theta}^\beta(s)}} \\& \nabla_\theta \left[r(s, a) + \frac{1}{\beta}\log \int_{\mathcal{S}} p(s'|s, a) e^{\beta V_{\pi_\theta}^\beta(s')} \text{d}s'\right] \text{d}a \\
    &\overset{(b)}{=} \frac{1}{\beta} \int_\mathcal{A} \nabla_\theta \pi_\theta(a|s) \frac{e^{\beta Q_{\pi_\theta}^\beta(s, a)}}{e^{\beta V_{\pi_\theta}^\beta(s)}} \text{d}a \\
    &+ \int_\mathcal{A} \frac{\pi_\theta(a|s) e^{\beta Q_{\pi_\theta}^\beta(s, a)}}{e^{\beta V_{\pi_\theta}^\beta(s)}} \frac{1}{\int_{\mathcal{S}} p(s'|s, a) e^{\beta V_{\pi_\theta}^\beta(s')} \text{d}s'} \\ &\int_{\mathcal{S}} p(s'|s, a) e^{\beta V_{\pi_\theta}^\beta(s')} \nabla_\theta V_{\pi_\theta}^\beta(s') \text{d}s' \text{d}a 
\end{aligned}
\end{equation*}
\begin{equation*}
    \begin{aligned}
    &= \frac{1}{\beta} \int_\mathcal{A} \nabla_\theta \pi_\theta(a|s) \frac{e^{\beta Q_{\pi_\theta}^\beta(s, a)}}{e^{\beta V_{\pi_\theta}^\beta(s)}} \text{d}a \\ &+ \int_\mathcal{A} \frac{\pi_\theta(a|s) e^{\beta Q_{\pi_\theta}^\beta(s, a)}}{e^{\beta V_{\pi_\theta}^\beta(s)}} \\ &\int_{\mathcal{S}} \frac{p(s'|s, a) e^{\beta V_{\pi_\theta}^\beta(s')}}{\int_{\mathcal{S}} p(s'|s, a) e^{\beta V_{\pi_\theta}^\beta(s')} \text{d}s'} \nabla_\theta V_{\pi_\theta}^\beta(s')  \text{d}s' \text{d}a \\
    &= \frac{1}{\beta}  \int_\mathcal{A} \nabla_\theta \pi_\theta(a|s) \frac{e^{\beta Q_{\pi_\theta}^\beta(s, a)}}{e^{\beta V_{\pi_\theta}^\beta(s)}} \text{d}a \\ &+ \int_\mathcal{A} \int_{\mathcal{S}} \pi^*(a|s) p^*(s'|s, a) \nabla_\theta V_{\pi_\theta}^\beta(s')  \text{d}s' \text{d}a \\
    &=  \frac{1}{\beta} \int_\mathcal{A} \nabla_\theta \pi_\theta(a|s) \frac{e^{\beta Q_{\pi_\theta}^\beta(s, a)}}{e^{\beta V_{\pi_\theta}^\beta(s)}} \text{d}a \\ &+ \int_{\mathcal{S}} p^*(s \to s', 1, \pi^*) \nabla_\theta V_{\pi_\theta}^\beta(s')  \text{d}s' \\
    &= \frac{1}{\beta} \int_\mathcal{A} \nabla_\theta \pi_\theta(a|s) \frac{e^{\beta Q_{\pi_\theta}^\beta(s, a)}}{e^{\beta V_{\pi_\theta}^\beta(s)}} \text{d}a \\ &+ \!\!\int_{\mathcal{S}} p^*(s \to s', 1, \pi^*) \frac{1}{\beta} \int_{a'}  \nabla_\theta \pi_\theta(a'|s')\frac{e^{\beta Q_{\pi_\theta}^\beta(s', a')}}{e^{\beta V_{\pi_\theta}^\beta(s')}} \text{d}a' \! \text{d}s' \\
    &+\!\!\int_{\mathcal{S}} \!\!p^*(s \to s', 1, \pi^*) \!\! \int_{\mathcal{S}} \!\!\! p^*(s' \to s'', 1, \pi^*) \nabla_\theta V_{\pi_\theta}^\beta(s'') \text{d}s'' \! \text{d}s'\\ 
    &\overset{(c)}{=} \frac{1}{\beta} \int_\mathcal{A} \nabla_\theta \pi_\theta(a|s) \frac{ e^{\beta Q_{\pi_\theta}^\beta(s, a)}}{ e^{\beta V_{\pi_\theta}^\beta(s)}} \text{d}a\\ & \!\!+ \int_{\mathcal{S}} \!\!p^*(s\to s', 1, \pi^*) \frac{1}{\beta} \int_{a'} \nabla_\theta \pi_\theta(a'|s') \frac{e^{\beta Q_{\pi_\theta}^\beta(s', a')}}{e^{\beta V_{\pi_\theta}^\beta(s')}} \text{d}a' \text{d}s' \\
    &+\int_{\mathcal{S}} p^*(s\to s'', 2, \pi^*) \nabla_\theta V_{\pi_\theta}^\beta(s'')  \text{d}s''\\
    & \vdots\\
    &=\!\! \int_{\mathcal{S}} \sum_{t=0}^T p^*(s \to s', t, \pi^*)  \frac{1}{\beta} \! \! \int_\mathcal{A} \nabla_\theta \pi_\theta(a|s) \frac{e^{\beta Q_{\pi_\theta}^\beta(s, a)}}{e^{\beta V_{\pi_\theta}^\beta(s)}} \text{d}a  \text{d}s'\!\!,
    \end{aligned}
\end{equation*}
where in (a) and (b) we used the Leibniz integral rule to exchange order of derivative and integration (Assumption 1), and in (c) we have used Fubini's theorem to exchange the order of integration (Assumption 2). Now we compute the entropic risk measure over $p_1(s)$,
\begin{equation*}
    \begin{aligned}
    \nabla_\theta J(\pi_\theta) &= \nabla_\theta \frac{1}{\beta} \log \int_\mathcal{S} p_1(s)e^{\beta V_{\pi_\theta}^\beta(s)} \text{d}s\\
    &\overset{(d)}{=} \int_\mathcal{S} \frac{p_1(s) e^{\beta V_{\pi_\theta}^\beta(s)}}{\int_\mathcal{S} p_1(s)e^{\beta V_{\pi_\theta}^\beta(s)} \text{d}s}  \nabla_\theta V_{\pi_\theta}^\beta(s) \text{d}s \\
    &= \int_\mathcal{S} p_1^*(s) \nabla_\theta V_{\pi_\theta}^\beta(s) \text{d}s\\
    &= \int_\mathcal{S} \int_{\mathcal{S}} \sum_{t=0}^T  p_1^*(s) p^*(s \to s', t, \pi^*)
    \end{aligned}
\end{equation*}
\begin{equation*}
    \begin{aligned}
    &\frac{1}{\beta} \int_\mathcal{A} \nabla_\theta \pi_\theta(a|s) \frac{e^{\beta Q_{\pi_\theta}^\beta(s, a)}}{e^{\beta V_{\pi_\theta}^\beta(s)}} \text{d}a  \text{d}s' \text{d}s \\
    &\overset{(e)}{=} \frac{1}{\beta} \int_{\mathcal{S}} \rho_{\pi}^*(s) \int_\mathcal{A} \nabla_\theta \pi_\theta(a|s) \frac{e^{\beta Q_{\pi_\theta}^\beta(s, a)}}{e^{\beta V_{\pi_\theta}^\beta(s)}} \text{d}a  \text{d}s,
    \end{aligned}
\end{equation*}
where in (d) we used the Leibniz integral rule to exchange derivative and integral (Assumption 1), and in (e) we used Fubini's theorem to exchange the order of integration (Assumption 2).
\end{proof}

\subsection{Proof of Risk-Sensitive Deterministic Policy Gradient Theorem}
This proof considers similar regularity assumptions as in \cite{silver2014deterministic}:

\paragraph{Assumption 1:} $p(s'|s, a)$, $\nabla_a p(s'|s, a)$, $\mu_\theta(s)$, $\nabla_\theta \mu_\theta(s)$, $r(s, a)$, $\nabla_a r(s, a)$, $p_1(s)$ are continuous in $s$, $a$, $s'$ and $\theta$.

\paragraph{Assumption 2:} there exists a $b$ and $L$ such that $\sup_s p_1(s) < b$, $\sup_{s,a,s'} p(s'|s, a) < b$, $\sup_{s,a} r(s, a) < b$, $\sup_{s,a,s'} ||\nabla_{a} p(s'|s, a)|| < L$, and $\sup_{s,a} ||\nabla_a r(s, a)|| < L$.

\paragraph{Assumption 3:} $\mathcal{A} = \mathbb{R}^m$ and $\mathcal{S}$ is a compact subset of $\mathbb{R}^d$.

In particular, assumption 1 implies that $V_{\mu_\theta}^\beta(s)$ and $\nabla_\theta V_{\mu_\theta}^\beta(s)$ are continuous functions of $\theta$ and $s$ and the compactness of $\mathcal{S}$ further implies that for any $\theta$, $||\nabla_\theta V_{\mu_\theta}^\beta(s)||$, $||\nabla_a Q_{\mu_\theta}^\beta(s, a)|_{a=\mu_\theta(s)}||$ and $||\nabla_\theta \mu_\theta(s)||$ are bounded functions of s. In addition, we define $p_1^*(s)$ as the initial distribution proportional to $p_1(s) e^{\beta V_{\mu_\theta}^\beta(s)}$, the distribution of $s'$ after transitioning for $t$ time steps from state $s$ using dynamics $p^*(s'|s,a)$ and deterministic policy $\mu_\theta(s)$ as $p^*(s\to s', t, \mu)$, and $\rho_\mu^*(s') := \int_{\mathcal{S}} \sum_{t=0}^T  p_1^*(s) p^*(s \to s', t, \mu) \text{d}s$. Now we present our result,
\renewcommand\thetheorem{3.2}
\begin{theorem}
(Risk-Sensitive Deterministic Policy Gradient Theorem). The gradient w.r.t. $J^\beta(\mu_\theta)$ is given by,
\begin{gather*}
    \nabla_\theta J^\beta(\mu_\theta) = \int_\mathcal{S} \rho_{\mu}^*(s) \nabla_\theta \mu_\theta(s) \nabla_a Q_{\mu_\theta}^\beta(s,a)|_{a = \mu_\theta(s)} \mathrm{d}s.
\end{gather*}
\end{theorem}
\begin{proof}
\begin{equation*}
    \begin{aligned}
        \nabla_\theta V_{\mu_\theta}^\beta(s) &= \nabla_\theta Q_{\mu_\theta}^\beta(s, \mu_\theta(s)) \\
        &= \nabla_\theta [r(s, \mu_\theta(s)) \\ & + \frac{1}{\beta} \log \int_{\mathcal{S}} p(s'|s, \mu_\theta(s)) e^{\beta V_{\mu_\theta}^\beta(s')} \text{d}s']\\ 
        &=  \nabla_\theta \mu_\theta(s) \nabla_a r(s, a)|_{a=\mu_\theta(s)} \\ &+ \frac{1}{\beta\int_{\mathcal{S}} p(s'|s, \mu_\theta(s)) e^{\beta V_{\mu_\theta}^\beta(s')} \text{d}s'} \\ &\nabla_\theta \int_{\mathcal{S}} p(s'|s, \mu_\theta(s)) e^{\beta V_{\mu_\theta}^\beta(s')} \text{d}s'\\ 
\end{aligned}
\end{equation*}
\begin{equation*}
    \begin{aligned}
        &\overset{(a)}{=} \nabla_\theta \mu_\theta(s) \nabla_a r(s, a)|_{a=\mu_\theta(s)} \\& + \frac{1}{\beta \int_{\mathcal{S}} p(s'|s, \mu_\theta(s)) e^{\beta V_{\mu_\theta}^\beta(s')} \text{d}s'}  \\ &\int_{\mathcal{S}} \nabla_\theta \mu_\theta(s) \nabla_a p(s'|s, a)|_{a=\mu_\theta(s)} e^{\beta V_{\mu_\theta}^\beta(s')} \text{d}s' \\
        &+ \int_{\mathcal{S}}  \frac{p(s'|s, \mu_\theta(s)) e^{\beta V_{\mu_\theta}^\beta(s')}}{\int_{\mathcal{S}} p(s'|s, \mu_\theta(s)) e^{\beta V_{\mu_\theta}^\beta(s')} \text{d}s'} \nabla_\theta V_{\mu_\theta}^\beta(s') \text{d}s' \\ 
        &= \nabla_\theta \mu_\theta(s)[\nabla_a r(s, a)|_{a=\mu_\theta(s)} \\ &+\frac{1}{\beta \int_{\mathcal{S}} p(s'|s, \mu_\theta(s)) e^{\beta V_{\mu_\theta}^\beta(s')} \text{d}s'} \\ & \int_{\mathcal{S}} \nabla_a p(s'|s, a)|_{a=\mu_\theta(s)} e^{\beta V_{\mu_\theta}^\beta(s')} \text{d}s' ]\\ 
        &+ \int_{\mathcal{S}} \frac{p(s'|s, \mu_\theta(s)) e^{\beta V_{\mu_\theta}^\beta(s')}}{\int_{\mathcal{S}} p(s'|s, \mu_\theta(s)) e^{\beta V_{\mu_\theta}^\beta(s')} \text{d}s'}   \nabla_\theta V_{\mu_\theta}^\beta(s') \text{d}s'\\ 
        &= \nabla_\theta \mu_\theta(s) \nabla_a [ r(s, a) \\ &+ \frac{1}{\beta} \log \int_{\mathcal{S}} p(s'|s, a) e^{\beta V_{\mu_\theta}^\beta(s')} \text{d}s' ]|_{a=\mu_\theta(s)} \\ &+ \int_{\mathcal{S}} p^*(s'|s, \mu_\theta(s)) \nabla_\theta V_{\mu_\theta}^\beta(s') \text{d}s' \\ 
        &= \nabla_\theta \mu_\theta(s) \nabla_a Q_{\mu_\theta}^\beta(s,a)|_{a=\mu_\theta(s)} \\ &+ \int_{\mathcal{S}} p^*(s'|s, \mu_\theta(s)) \nabla_\theta V_{\mu_\theta}^\beta(s') \text{d}s'\\
        &= \nabla_\theta \mu_\theta(s) \nabla_a Q_{\mu_\theta}^\beta(s,a)|_{a=\mu_\theta(s)} \\ &+ \int_{\mathcal{S}} p^*(s \to s', 1, \mu_\theta) \nabla_\theta V_{\mu_\theta}^\beta(s') \text{d}s' \\
        &= \nabla_\theta \mu_\theta(s) \nabla_a Q_{\mu_\theta}^\beta(s,a)|_{a=\mu_\theta(s)} \\
        & + \int_{\mathcal{S}} p^*(s \to s', 1, \mu_\theta) \nabla_\theta \mu_\theta(s) \nabla_a Q_{\mu_\theta}^\beta(s,a)|_{a=\mu_\theta(s)} \text{d}s' \\ 
        &+\!\! \!\int_{\mathcal{S}}\!\! p^*(s \to s', 1, \mu_\theta) \!\!\!\int_{\mathcal{S}} \!\!p^*(s' \to s'', 1, \mu_\theta) \nabla_\theta V_{\mu_\theta}^\beta(s'') \text{d}s'' \!\text{d}s' \\ 
        &\overset{(b)}{=} \nabla_\theta \mu_\theta(s) \nabla_a Q_{\mu_\theta}^\beta(s,a)|_{a=\mu_\theta(s)} \\
        & + \int_{\mathcal{S}} p^*(s \to s', 1, \mu_\theta) \nabla_\theta \mu_\theta(s) \nabla_a Q_{\mu_\theta}^\beta(s,a)|_{a=\mu_\theta(s)} \text{d}s' \\ 
        &+ \int_{\mathcal{S}} p^*(s \to s', 2, \mu_\theta) \nabla_\theta V_{\mu_\theta}^\beta(s') \text{d}s' \\ 
        & \vdots \\ 
        &=\!\!\! \!\int_{\mathcal{S}} \! \sum_{t=0}^T p^*(s \to s', t, \mu_\theta) \!\nabla_\theta \mu_\theta(s')  \! \nabla_a \!Q_{\mu_{\theta}}^\beta \!(s', a)|_{a = \mu_\theta(s')} \text{d}s'\!\!,
    \end{aligned}
\end{equation*}
where in (a) we used the Leibniz integral rule to exchange order of derivative and integration (Assumption 1), and in (b) we used Fubini's theorem to exchange the order of integration (Assumption 2). Now we compute the entropic risk measure over $p_1(s)$,
\begin{equation*}
    \begin{aligned}
    \nabla_\theta J^\beta(\mu_\theta) &= \nabla_\theta \frac{1}{\beta} \log \int_\mathcal{S} p_1(s) e^{\beta V_{\mu_\theta}^\beta(s)} \text{d}s\\
    &\overset{(c)}{=} \int_\mathcal{S} \frac{p_1(s) e^{\beta V_{\mu_\theta}^\beta(s)}}{\int_\mathcal{S} p_1(s) e^{\beta V_{\mu_\theta}^\beta(s)} \text{d}s}  \nabla_\theta V_{\mu_\theta}^\beta(s) \text{d}s\\
    &= \int_\mathcal{S} p_1^*(s) \nabla_\theta V_{\mu_\theta}^\beta(s) \text{d}s\\
    &= \int_\mathcal{S} \int_{\mathcal{S}} \sum_{t=0}^T  p_1^*(s) p^*(s \to s', t, \mu_\theta) \\ & \nabla_\theta \mu_\theta(s') \nabla_a Q_{\mu_{\theta}}^\beta(s', a)|_{a = \mu_\theta(s')}  \text{d}s' \text{d}s \\
    &\overset{(d)}{=} \int_{\mathcal{S}} \rho_{\mu}^*(s)  \nabla_\theta \mu_\theta(s) \nabla_a Q_{\mu_{\theta}}^\beta(s, a)|_{a = \mu_\theta(s)}  \text{d}s,
    \end{aligned}
\end{equation*}
where in (c) we used the Leibniz integral rule to exchange derivative and integral (Assumption 1), and in (d) we used Fubini's theorem to exchange the order of integration (Assumption 2).
\end{proof}

\subsection{Proof of Policy Improvement for the Off-Policy Risk-Sensitive Deterministic Setting}
We consider the off-policy performance objective for a deterministic policy, soft-value functions of a target deterministic policy weighted by the state distribution of the behavior policy,
\begin{equation*}
    J_\beta(\mu_\theta) = \sum_{s \in \mathcal{S}} \rho_b(s) V_{\mu_\theta}^\beta(s), 
\end{equation*}
where $\rho_b$ is the state distribution under the behavior policy $b(a|s)$. Similarly to \citep{degris2012off}, we propose an approximation of the gradient:
\begin{equation*}
    \nabla_\theta J_b^\beta(\mu_\theta) \approx \sum_{s\in S} \rho_b(s) \nabla_\theta \mu_\theta(s) \nabla_a Q_{\mu_\theta}^\beta(s,a)|_{a = \mu_\theta(s)} := g(\mu_\theta)
\end{equation*}
This approximation drops a term that depends on $\nabla_\theta Q_{\mu_\theta}(s, a)$. We show that this approximation results in policy improvement when using a tabular representation for the policy.
\renewcommand\thetheorem{3.3}
\begin{theorem}
(Deterministic Policy Improvement). Given a policy $\mu$ with parameters $\theta$ and a linear tabular representation, let $\theta' = \theta +\alpha g(\mu_\theta)$. Then there exist an $\epsilon$ such that for all $\alpha < \epsilon$:
\begin{gather*}
    V_{\mu_{\theta}}^\beta(s) \leq V_{\mu_{\theta'}}^\beta(s)
\end{gather*}
for all $s \in S$.
\end{theorem}
\begin{proof}
Notice first that for any $s$, the gradient $\nabla_\theta \mu_\theta(s) \nabla_a Q_{\mu_\theta}^\beta(s,a) |_{a=\mu_\theta(s)}$ is the direction that increase $Q_{\mu_{\theta}}(s,\mu_\theta(s))$ for a fixed soft-value function $Q_{\mu_\theta}$. For an appropriate step size $\alpha$, we can guarantee that
\begin{equation*}
    J_\beta(\mu_\theta) = \sum_{s \in \mathcal{S}} \rho_b(s) Q_{\mu_\theta}^\beta(s, \mu_\theta(s)) \leq \sum_{s \in \mathcal{S}} \rho_b(s) Q_{\mu_\theta}^\beta(s,\mu_{\theta'}(s)) 
\end{equation*}
Given that $\mu_\theta$ has a tabular representation, we have that
\begin{equation*}
\begin{aligned}
    V_{\mu_\theta}^\beta(s) \! = \!Q_{\mu_\theta}^\beta(s,\mu_\theta(s)) \!\\  \leq \! Q_{\mu_\theta}^\beta(s,\mu_{\theta'}(s))\! =\! \mathbb{E}_{a \sim \mu_{\theta'}}[Q_{\mu_\theta}^\beta(s,a)], 
\end{aligned}
\end{equation*}
because the policy can be updated independently for each state with separate weights for each state. Hence, for any $s \in \mathcal{S}$ and $\beta > 0$:
\begin{equation*}
    \begin{aligned}
    e^{\beta V_{\mu_\theta}^\beta(s)} &\leq e^{\beta Q_{\mu_\theta}^\beta(s,\mu_{\theta'}(s))} \\
    &= \mathbb{E}_{a \sim \mu_{\theta'}(s)} \mathbb{E}_{r,s' \sim p(\cdot|s,a)}[e^{\beta r} e^{\beta V_{\mu_\theta}^\beta(s')}] \\
    &\leq \mathbb{E}_{a \sim \mu_{\theta'}(s)} \mathbb{E}_{r,s' \sim p(\cdot|s,a)}[\\ &e^{\beta r} \mathbb{E}_{a' \sim \mu_{\theta'}(s')} [e^{\beta Q_{\mu_\theta}^\beta(s', a')}]] \\ 
    & \vdots \\
    & = \mathbb{E}_{a \sim \mu_{\theta'}(s)}[e^{\beta Q_{\mu_{\theta'}}^\beta(s,a)}] \\
    & = e^{\beta V_{\mu_{\theta'}}^\beta(s)}.
    \end{aligned}
\end{equation*}
Hence, we can conclude that $V_{\mu_{\theta}}^\beta(s) \leq V_{\mu_{\theta'}}^\beta(s)$ for all $s\in \mathcal{S}$. The proof for $\beta < 0$ follows similar arguments.
\end{proof}

\subsection{Proof of Policy Improvement for the Off-Policy Risk-Sensitive Stochastic Setting}
We now consider the off-policy performance objective for a stochastic policy, soft-value functions of a target stochastic policy weighted by the state distribution of the behavior policy,
\begin{equation*}
    J_b^\beta(\pi_\theta) = \frac{1}{\beta}\sum_{s \in \mathcal{S}} \rho_b(s) e^{\beta V_{\pi_\theta}^\beta(s)} 
\end{equation*}
where $\rho_b$ is the state distribution under the behavior policy $b(a|s)$. Similarly to \citep{degris2012off}, we propose an approximation of the gradient:
\begin{equation*}
    \nabla_\theta J_b^\beta(\pi_\theta) \! \approx \! \frac{1}{\beta} \sum_{s\in S} \rho_b(s) \sum_{a \in A} \nabla_\theta \pi_\theta (a|s)e^{\beta Q_{\pi_\theta}^\beta(s,a)} \!\! := \! g(\pi_\theta)
\end{equation*}
This approximation drops a term that depends on $\nabla_\theta Q_{\pi_\theta}(s, a)$. We show that this approximation results in policy improvement when using a tabular representation for the policy.
\renewcommand\thetheorem{3.4}
\begin{theorem}
(Stochastic Policy Improvement). Given a policy $\pi$ with parameters $\theta$ and a tabular representation, let $\theta' = \theta +\alpha g(\pi_\theta)$. Then there exist an $\epsilon > 0$ such that for all positive $\alpha < \epsilon$,
\begin{gather*}
    V_{\pi_{\theta}}^\beta(s) \leq V_{\pi_{\theta'}}^\beta(s)
\end{gather*}

for all $s \in S$.
\end{theorem}
\begin{proof}
Notice first that for any $s$, the gradient $\sum_{a \in \mathcal{A}} \nabla_\theta \pi_\theta (a|s)e^{\beta Q_{\pi_\theta}^\beta(s,a)}$ is the direction that increases $\sum_{a \in \mathcal{A}} \pi_\theta (a|s)e^{\beta Q_{\pi_\theta}^\beta(s,a)}$ for a fixed soft-value function $Q_{\pi_\theta}$. For an appropriate step size $\alpha$, we can guarantee that
\begin{equation*}
\begin{aligned}
    J_\beta(\pi_\theta) = \sum_{s \in \mathcal{S}} \rho_b(s) \sum_{a\in\mathcal{A}} \pi_\theta(a|s)e^{\beta Q_{\pi_\theta}^\beta(s,a)} \\ \leq \sum_{s \in \mathcal{S}} \rho_b(s) \sum_{a\in\mathcal{A}} \pi_{\theta'}(a|s)e^{\beta Q_{\pi_\theta}^\beta(s,a)}
\end{aligned}
\end{equation*}
Given that $\pi_\theta$ has a tabular representation, we have that
\begin{equation*}
    \begin{aligned}
    e^{\beta V_{\pi_\theta}^\beta(s)} = \sum_{a\in\mathcal{A}} \pi_\theta(a|s) e^{\beta Q_{\pi_\theta}^\beta(s,a)} \\ \leq \sum_{a \in \mathcal{A}} \pi_{\theta'}(a|s)e^{\beta Q_{\pi_\theta}^\beta(s,a)} = \mathbb{E}_{a \sim \pi_{\theta'}}[e^{\beta Q_{\pi_\theta}^\beta(s,a)}] 
    \end{aligned}
\end{equation*}
because the policy can be updated independently for each state with separate weights for each state. Hence, for any $s \in \mathcal{S}$ and $\beta > 0$:
\begin{equation*}
    \begin{aligned}
    e^{\beta V_{\pi_\theta}^\beta(s)} &\leq \mathbb{E}_{a \sim \pi_{\theta'}}[ e^{\beta Q_{\pi_{\theta}}^\beta(s,a)}] \\
    &= \mathbb{E}_{a \sim \pi_{\theta'}}[\mathbb{E}_{r,s' \sim p(\cdot|s,a)}[e^{\beta r}  e^{\beta V_{\pi_\theta}^\beta(s')}]] \\
    &\leq \mathbb{E}_{a \sim \pi_{\theta'}}[\mathbb{E}_{r, s' \sim p(\cdot|s,a)}[e^{\beta r} \mathbb{E}_{a' \sim \pi_{\theta'}} [e^{\beta Q_{\pi_\theta}^\beta(s', a')}]]] \\ 
    & \vdots \\
    & = \mathbb{E}_{a \sim \pi_{\theta'}}[e^{\beta Q_{\pi_{\theta'}}^\beta (s,a)}] \\
    & = e^{\beta V_{\pi_{\theta'}}^\beta(s)}.
    \end{aligned}
\end{equation*}

Hence, we can conclude that $V_{\pi_{\theta}}^\beta(s) \leq V_{\pi_{\theta'}}^\beta(s)$ for all $s\in \mathcal{S}$. The proof for $\beta < 0$ follows similar arguments.
\end{proof}

\section{Critic Gradient Stabilization}\label{sec:batch_normalization}
\renewcommand\thefigure{7} 
\begin{figure}[t]
    \centering
     \begin{subfigure}[b]{0.45\textwidth}
    \includegraphics[width=\textwidth]{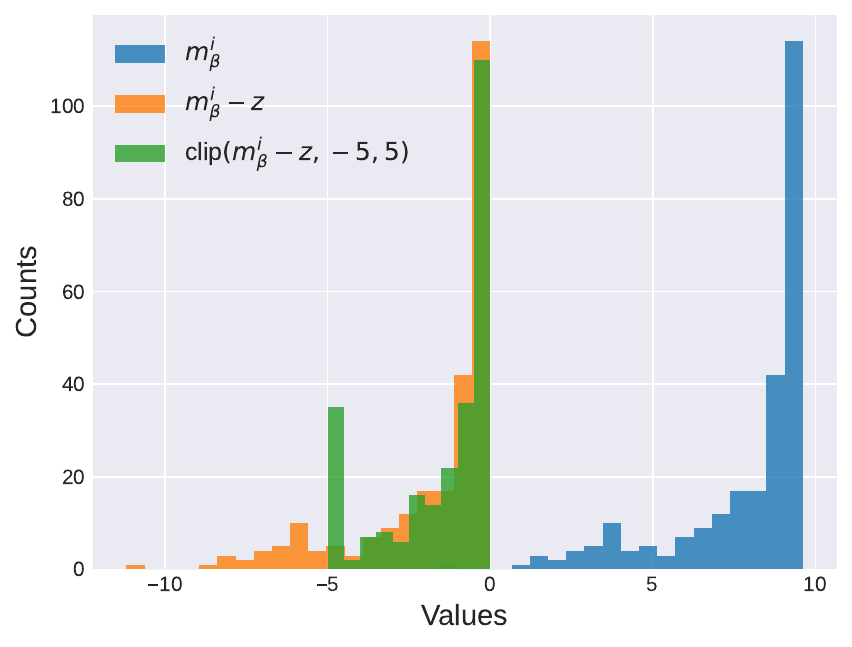}
     \end{subfigure} 
     \caption{\textbf{Histograms for stabilizing operations on a mini-batch.} We compute $m_i$ values and $z$ for a mini-batch of size 256. Histogram of: $m_i$ values (\emph{blue}), $m_i$ values after subtracting $z$ (\emph{orange}), $m_i$ values after subtraction by $z$ and clipping (\emph{green}). 
        }
     \label{fig:histogram}
\end{figure}

In this section, we present our derivation of a constant $z$ that sufficiently normalizes the gradient over the current mini-batch, and our clipping mechanism.  We achieve this by defining,
\begin{equation}
    \!\!\!\!m_{\beta}^i \triangleq Q_\psi(s_t^i, a_t^i) + \max(Q_\psi(s_t^i, a_t^i), \beta r_t^i +Q_{\psi'}(s_{t+1}^i))
\end{equation}
where $\{(s_t^i, a_t^i, s_{t+1}^i, r_t^i)\}_{i=1}^N$ are samples from the current mini-batch.  Then, we set the constant $z$ to $\max_i m_{\beta}^i$ when $\beta > 0$, and to $\min_i m_{\beta}^i$ when $\beta < 0$.  This yields a gradient proportional to $\nabla J^i_{Q}(\psi)$: 
\begin{equation}
      e^{m_{\beta}^i - z} f(Q_\psi(s_t^i, a_t^i), \beta r_t^i +Q_{\psi'}(s_{t+1}^i))
        \nabla Q_\psi(s_t^i, a_t^i)
\end{equation}
Finally, we avoid any numerical overflow (underflow) by clipping the exponent of the leading term in the interval $[-c,c]$ as: $e^{\text{clip}({m_{\beta}^i - z},-c,c)}$ to compute $\nabla J_{\text{clip},Q}^i$.  In \FIG \ref{fig:histogram}, we illustrate the operations performed to the leading terms in a mini-batch during learning in the CartPole environment.

\section{Implementation Details}\label{sec:Mujoco}
\subsection{CartPole environment}
Table \ref{hyperparameter} lists the hyperparameters we used for the CartPole environment.

\begin{table}[ht]
  \caption{Hyperparameters for CartPole environment}
  \label{hyperparameter}
  \centering
  \begin{tabular}{ll}
    \toprule
    \multicolumn{1}{c}{Schedule details}   \\
    Env steps before training  & 10000 steps   \\
    Env steps per epoch  & 1000 steps  \\
    \toprule
    \multicolumn{1}{c}{Network details}   \\
    Discount factor  & 0.99   \\
    Soft target update & 0.005   \\
    Experience buffer $\Dcal_{\text{env}}$ & 10,000  \\
    Value Network  & 2 hidden layers size 128  \\
    Network optimizer & AdamW \\ 
    Non-linear layers & ReLU \\
    Learning rate & 0.0003 \\
    \bottomrule
    \multicolumn{1}{c}{Policy details}   \\
    $\epsilon$ exploration value & 0.1 \\
    \toprule
    \multicolumn{1}{c}{Computational details per run}   \\
    CPU & 1 Altus XE2214GT   \\
    RAM & 5GBs   \\
    GPU & 1 Tesla V100S 32GB  \\
    Running time & 2 hours   \\
    \bottomrule
  \end{tabular}
\end{table}

\subsection{Mujoco environments Hyperparameters}
Each task is run for 1 million environment steps with evaluations every 10000 time steps. Agents are tested for 20 episodes per evaluation. We perform 10 runs of each algorithm with different random seeds and report average and standard deviation. For MG and MVPI, we use the implementations in \citep{luo2023alternative} and follow the same hyperparameters suggested by the authors. We implement R-AC \citep{noorani2023exponential} on top of the TD3, and select its risk parameter from $\{-1,-0.1, -0.05 , -0.01\}$. For rsEAC, we select $\beta$ from $\{-0.05, -0.01, -0.0025\}$. We use the same network architectures and learning rates for all algorithms. Table \ref{hyperparameter_muj} lists the hyperparameters used for rsEAC and all actor-critic algorithms in risk-aware MuJoCo benchmarks.

\begin{table}[t]
  \caption{Hyperparameters for risk-aware MuJoCo}
  \label{hyperparameter_muj}
  \centering
  \begin{tabular}{ll}
    \toprule
    \multicolumn{1}{c}{Schedule details}   \\
    Env steps before training  & 5000 steps   \\
    Env steps per epoch  & 1000 steps  \\
    \toprule
    \multicolumn{1}{c}{Network details}   \\
    Discount factor  & 0.99   \\
    Soft target update & 0.005   \\
    Experience buffer $\Dcal_{\text{env}}$ & 1,000,000  \\
    Actor Architecture  & 2 hidden layers size 256  \\
    Critic Architecture  & 2 hidden layers size 256  \\
    Network optimizer & Adam \\ 
    Non-linear layers & ReLU \\
    Exponential gradient clip & 5 \\
    Learning rate & 0.0003 \\
    Policy noise & 0.2 \\
    Noise clip & 0.5 \\
    Iterations before policy update & 2 \\
    \toprule
    \multicolumn{1}{c}{Computational details per run}   \\
    CPU & 1 Altus XE2214GT   \\
    RAM & 5GBs   \\
    GPU & 1 Tesla V100S 32GB  \\
    Running time & 8 hours   \\
    \bottomrule
  \end{tabular}
\end{table}


\newpage
\section{Pseudocode of rsEAC}\label{sec:pseudo}
This section contains the pseudocode for our algorithm, rsEAC.
\begin{algorithm}[ht]
   \caption{rsEAC}
   \label{alg:example}
\begin{algorithmic}
   \STATE Initialize networks, parameters and buffers:
   \FOR{$t=1$ {\bfseries to} $T$}
   \STATE Select action with exploration noise:
    \STATE $a_t \sim \mu_\theta(s_t) + \epsilon$, $\epsilon \sim \mathcal{N}(0,\sigma^2)$ 
    \STATE Sample next state and reward from environment: 
    \STATE $s_{t+1}, r_t \sim p(s_{t+1}, r_t|s_t,a_t)$
    \STATE Store transition to experience buffer:
    \STATE $\Dcal \leftarrow \Dcal \cup \{(s_t,a_t, s_{t+1}, r_t)$\} 
    \STATE Sample mini-batch of size $N$ from $\mathcal{D}$:
   \STATE $\{(s_t^i, a_t^i, s_{t+1}^i, r_t^i)\}_{i=1}^N \sim \Dcal$ 
   \STATE Sample next actions:
   \STATE $a_{t+1}^i \sim \mu_\theta(s_{t+1}^i) + \epsilon^i$, $\epsilon^i \sim \text{clip}(\mathcal{N}(0,\sigma^2), -c^*, c^*)$
   \STATE Compute target:
   \STATE $y^i_t = \beta r_t^i + \gamma \min_{j=1,2} Q_{\psi_j'}(s_{t+1}^i, a_{t+1}^i)$
   \STATE Compute our gradient $\nabla \!J_{\text{clip},Q}^i(\psi)$:
   \STATE $=  e^{\text{clip}(m_{\beta}^i - z, -c, c)} f(Q_\psi(s_t^i, a_t^i), y_t^i)
        \nabla  Q_\psi(s_t^i, a_t^i)$ 
    \STATE Update critics:
   \STATE $\psi \leftarrow \psi -N^{-1} \sum_i \nabla J_{\text{clip},Q}^i(\psi)$  
   \IF{$t \text{ mod } d$}
   \STATE Compute deterministic policy gradient $\nabla J(\theta)$:
   \STATE $= N^{-1} \sum_i \nabla_{\theta} \mu_\theta(s_t^i) \nabla_{a_t} \frac{1}{\beta} Q_{\psi_1}(s_t^i, a_t^i)|_{a_t^i = \mu_\theta(s_{t}^i)} $ 
   \STATE Update actor:
   \STATE $\theta \leftarrow \theta -\nabla J(\theta)$  
   \STATE Update target networks:
   \STATE $\theta' \leftarrow \tau \theta + (1-\tau)\theta'$
   \STATE $\psi_i' \leftarrow \tau \psi_i + (1-\tau)\psi_i'$
   \ENDIF
   \ENDFOR
\end{algorithmic}
\end{algorithm}
\end{document}